\documentclass[letterpaper]{article} % DO NOT CHANGE THIS
\usepackage{aaai2026}  % DO NOT CHANGE THIS
\usepackage{times}  % DO NOT CHANGE THIS
\usepackage{helvet}  % DO NOT CHANGE THIS
\usepackage{courier}  % DO NOT CHANGE THIS
\usepackage[hyphens]{url}  % DO NOT CHANGE THIS
\usepackage{graphicx} % DO NOT CHANGE THIS
\usepackage{natbib}  % DO NOT CHANGE THIS AND DO NOT ADD ANY OPTIONS TO IT
\usepackage{caption} % DO NOT CHANGE THIS AND DO NOT ADD ANY OPTIONS TO IT
\usepackage{algorithm}
\usepackage{algorithmic}
\usepackage{subcaption}
\usepackage{amssymb}
\usepackage{amsmath}
\usepackage{multirow}
\usepackage{booktabs}

\newtheorem{lemma}{Lemma}
\newtheorem{proof}{Proof}

\usepackage{newfloat}
\usepackage{listings}
\DeclareCaptionStyle{ruled}{labelfont=normalfont,labelsep=colon,strut=off} % DO NOT CHANGE THIS
\floatstyle{ruled}
\newfloat{listing}{tb}{lst}{}
\floatname{listing}{Listing}
\title{From Attribution to Action: Jointly ALIGNing Predictions and Explanations}

\author {
    Dongsheng Hong\textsuperscript{\rm 1} \equalcontrib,
    Chao Chen\textsuperscript{\rm 2} \equalcontrib,
    Yanhui Chen\textsuperscript{\rm 1},
    Shanshan Lin\textsuperscript{\rm 1},
    Zhihao Chen\textsuperscript{\rm 1},
    Xiangwen Liao\textsuperscript{\rm 1}\thanks{Corresponding author.}
}
\affiliations {
    \textsuperscript{\rm 1}
    College of Computer and Data Science, Fuzhou University, China\\
    \textsuperscript{\rm 2}
    School of Computer Science and Technology, Harbin Institute of Technology (Shenzhen), China\\
     \{2501016004, 231020061, 231016006, n180320046, liaoxw\}@fzu.edu.cn,
     cha01nbox@gmail.com
}

\begin{document}

\maketitle

\begin{abstract}
Explanation-guided learning (EGL) has shown promise in aligning model predictions with interpretable reasoning, particularly in computer vision tasks. 
However, most approaches rely on external annotations or heuristic-based segmentation to supervise model explanations, which can be noisy, imprecise and difficult to scale. 
In this work, we provide both empirical and theoretical evidence that low-quality supervision signals can degrade model performance rather than improve it. 
In response, we propose ALIGN, a novel framework that jointly trains a classifier and a masker in an iterative manner. 
The masker learns to produce soft, task-relevant masks that highlight informative regions, while the classifier is optimized for both prediction accuracy and alignment between its saliency maps and the learned masks.
By leveraging high-quality masks as guidance, ALIGN improves both interpretability and generalizability, showing its superiority across various settings.
Experiments on the two domain generalization benchmarks, VLCS and Terra Incognita, show that ALIGN consistently outperforms six strong baselines in both in-distribution and out-of-distribution settings. 
Besides, ALIGN also yields superior explanation quality concerning sufficiency and comprehensiveness, highlighting its effectiveness in producing accurate and interpretable models.

% Explanation-guided learning (EGL) has shown promise in aligning model predictions with interpretable reasoning, particularly in computer vision tasks. 
% However, most approaches rely on external annotations or heuristic-based segmentations to supervise model explanations, which limits scalability and precision. 
% We practically and theoretically justify that imprecise signals could harm, not benefit, the model's performance.
% Thus, we propose a novel framework, namely ALIGN, that iteratively trains a classifier and a masker to encourage the model's heatmaps to focus on relavant regions. 
% The masker highlights soft object masks optimized to align with model saliency maps. 
% The classifier in turns take classification accuracy and saliency alignment into accounct. 
% Experiments on the VLCS and Terra Incognita benchmarks demonstrate consistent improvements in both in-distribution and out-of-distribution accuracy, AUC, and explanation quality metrics such as sufficiency and comprehensiveness over existing baselines, demonstrating the effectiveness of ALIGN. 
\end{abstract}

%     \link{Code}{https://anonymous.4open.science/r/align-2E3B/}

\section{Introduction}

To provide transparent and trustworthy explanations for the decisions made by deep neural networks, 
% Explainable AI (XAI) techniques like Grad-CAM \cite{selvaraju2020gradcam} generate post-hoc saliency maps that highlight the input regions most influential to the model's predictions. 
% However, such post-hoc XAI methods are applied only \textit{after} model training and do not influence the model during learning. 
% In contrast, 
\textbf{Explanation-Guided Learning} (EGL) \cite{guidedlearning} integrates explanation signals (e.g., human-provided \textit{masks}) into the training process to align model reasoning with interpretable semantics. 
These masks typically highlight regions of interest that correspond to task-relevant, informative components in the input (e.g., objects or salient structures). 
For example, CARE \cite{zhuang2019care}, GRADIA \cite{gao2022gradia}, and MAGI \cite{zhang2023magi} penalize attributions to irrelevant regions based on human-annotated masks, thus encouraging the model to focus on relevant features and make more interpretable decisions. 

\begin{figure}[t]
    \centering
    \includegraphics[width=0.9\linewidth]{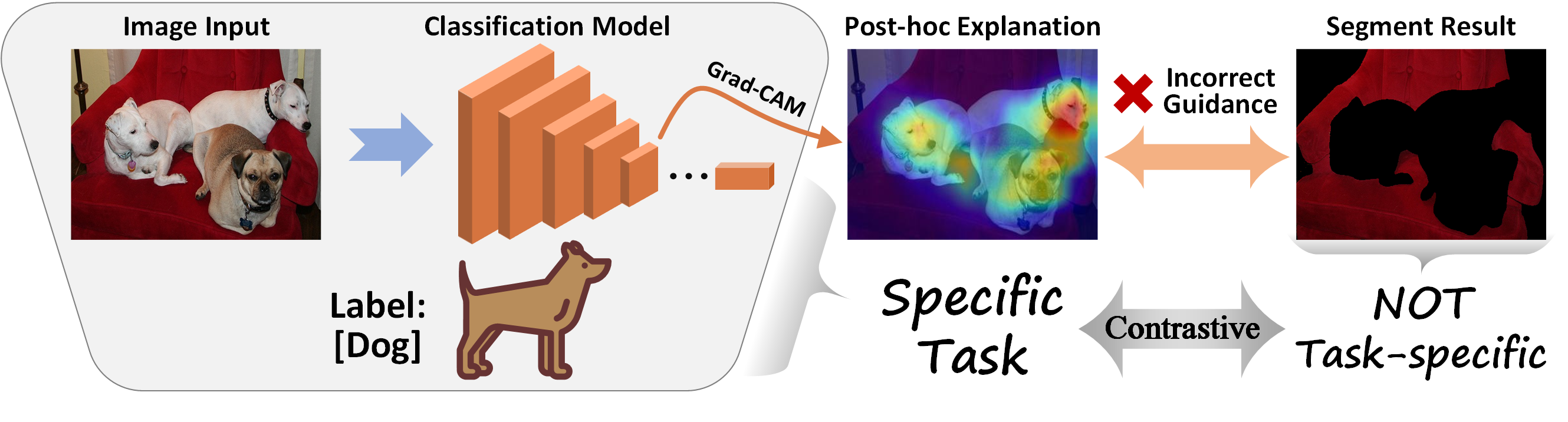}
    \caption{Segmentations as guidance may emphasize background, limiting reliability for task-specific explanations.}
    \label{fig:insight}
\end{figure}

Existing EGL methods rely heavily on \textit{manual annotations}, which are labor-intensive and potentially inaccurate. 
Recent approaches \cite{li2023dre,guesmi2024asgt} attempt to eliminate manual supervision by enforcing explanation consistency during training.
Leveraging segmentation results as guidance can improve explanation quality, but these upstream segmentation models are typically \textbf{not task-specific}, and may highlight irrelevant regions, yielding misleading explanations. 
As illustrated in Fig.~\ref{fig:insight}, although the image contains dogs, the segmentation mainly captures the surrounding environment rather than the target objects. 
This highlights the need for a \textbf{task-aware} masker that produces reliable, semantically aligned guidancen.
Moreover, prior work has primarily emphasized \textit{empirical} performance, with limited \textit{theoretical} insight into their generalization behavior.

In this paper, 
we first revisit the impact of mask quality on EGL by conducting a preliminary experiment, revealing that imprecise or low-quality masks can hinder prediction accuracy.
We further reinforce the need for high-quality, task-relevant masks through a theoretical analysis under the \textit{Probably Approximately Correct} (PAC) learning framework. 
Specifically, we show that better mask quality leads to tighter generalization bounds under domain shifts (e.g., under out-of-distribution settings), as well as lower in-distribution errors. 
Building on these insights, we propose \textbf{Attribution-Learning Iterative Guidance Network} (ALIGN), a novel framework that jointly trains a mask generator (termed \textit{masker}) and a \textit{classifier} in an iterative manner. 
Instead of relying on costly and potentially noisy annotations, ALIGN uses a learnable masker to produce soft masks that highlight semantically relevant regions of the input. Simultaneously, the classifier is optimized not only for predictive accuracy but also to align its own saliency maps with the generated masks. By explicitly guiding the model on which features to attend to, ALIGN enhances both interpretability and generalizability.

% In this paper, 
% we first investigate how mask quality affects model performance through a preliminary experiment, revealing that imprecise or low-quality masks can hinder prediction accuracy. 
% We further justify the need for high-quality, task-relevant masks through theoretical analysis under the \textit{Probably Approximately Correct} (PAC) framework.
% We show that better masks can lead to tighter discrepancy bounds under domain shifts, e.g., out-of-distribution cases.
% % Our analysis offers generalization guarantees for out-of-distribution performance and highlights how explanation-based constraints improve upon traditional empirical risk minimization.
% % 
% Furthermore, we propose \textbf{Attribution-Learning Iterative Guidance Network} (ALIGN), a novel framework that iteratively trains a masker and a predictive model.
% Rather than relying on costly and potentially imprecise human-annotated masks, 
% the masker learns to produce soft masks that highlight semantically relevant regions in the input. 
% Meanwhile, the classifier is trained not only for prediction accuracy, but also to align its saliency map with the generated masks. 
% % Unlike adversarial training frameworks \cite{guesmi2024asgt} aiming to stabilize attributions under perturbations, 
% ALIGN actively instructs the model \textit{which} features should be considered, enhancing interpretability without compromising accuracy. 
% % generates semantically meaningful masks and guides models to focus on these regions. 

Beyond theoretical justification, 
we conduct comprehensive experiments on two standard domain generalization benchmarks, VLCS and Terra Incognita. 
Compared with six state-of-the-art methods, including DRE \cite{li2023dre}, SGDrop \cite{bertoin2024sgdrop}, and SGT \cite{ismail2021sgt}, 
ALIGN achieves superior predictive performance on both in-distribution and out-of-distribution data.
Furthermore, ALIGN produces more meaningful explanations regarding sufficiency and comprehensiveness. 
Qualitative visualizations and extensive ablation studies further demonstrate that ALIGN yields robust, interpretable predictions.

In summary, the contributions of this paper are threefold:
\begin{itemize}
    \item 
    We revisit the role of mask quality in EGL, providing both \textit{empirical} evidence and \textit{theoretical} justification that high-quality masks enhance generalization, while poor masks degrade predictive performance.

    \item 
    We propose \textbf{ALIGN}, a novel annotation-free framework that jointly trains a masker and a classifier, aligning model attributions with learned masks to promote interpretability and generalizability.

    \item 
    Extensive experiments on domain generalization benchmarks demonstrate that ALIGN achieves superior accuracy and explanation quality compared to sota methods, which is supported by visualizations and ablation studies.

\end{itemize}

\section{Related Work}

Depending on the source of supervision used for explanation alignment, Explanation-Guided Learning (EGL) approaches can be broadly categorized into human-annotated and annotation-free methods.

\textit{Human-Annotated Supervision.}
Methods such as CARE \cite{zhuang2019care}, GRADIA \cite{gao2022gradia}, RES \cite{gao2022res}, and MAGI \cite{zhang2023magi} use human-labeled masks or saliency cues to align model attributions with semantically meaningful regions. 
% VISFIS \cite{ying2022visfis} applies similar principles to VQA tasks. 
While effective, these methods rely on costly manual annotations, limiting scalability.

\textit{Annotation-Free Training.}
To avoid manual supervision, methods like SGT \cite{ismail2021sgt}, SMOOT \cite{karkehabadi2024smoot}, and DRE \cite{li2023dre} enforce consistency between explanations and predictions using gradient-based masking or stability constraints. However, these methods could rely on in-sentimantical regions and lack theoretical guarantees.

Besides images, EGL has been extended to other modularities. For \textit{graphs}, GazeGNN \cite{wang2024gazegnn}, GNES \cite{gao2021gnes}, and GG-NES \cite{etemadyrad2024ggnes} align attention with node importance to enhance robustness and interpretability. 
For \textit{texts}, some work \cite{li2023symbolic,li2022explanations} leverage LLM-generated rationales to train smaller models, showing promise for scalable, explanation-aware training.
More relevant works are discussed in \ref{appendix:related_works}.

\section{Preliminaries}
% We first introduce two key concepts: Explainable Machine Learning (XML), and Explanation-Guided Learning (EGL).

\subsection{Explainable machine learning \& Grad-CAM}
Given an input-label pair $(x, y)$, where $x \in \mathbb{R}^d$ and $y\in Y$ denote the input features and the corresponding class label, respectively, 
a classifier $f_\theta: \mathbb{R}^d \rightarrow \mathbb{R}^{|Y|}$ is parameterized by $\theta$. For clarity, we omit $\theta$ and write $f(x)$. 
The model's predicted probability for class $y$ is denoted by $f_y(x)$.

% The primary goal of explainable machine learning is to find a importance evaluator $\Phi: \mathbb{R}^d \rightarrow \mathbb{R}^d$. 
% For each input $x$, this evaluator outputs the importance $\Phi_y(x)$ with respect to the classifier's prediction for class $y$.
% For image tasks, $\Phi$ calculates the contribution of each pixel to the classifier's prediction.
% To illustrate this, we draw on Grad-CAM \cite{selvaraju2020gradcam}, a method widely used in computer vision, to demonstrate how $\Phi_y(x)$ is computed.

The goal of explainable machine learning (XML) is to compute an importance map $\Phi: \mathbb{R}^d \rightarrow \mathbb{R}^d$ that reflects the relevance of each input dimension (e.g., pixel) to the model's prediction. 
Specifically, $\Phi_y(x)$ estimates the contribution of each element in $x$ to the prediction of class $y$.

In this paper, we adopt Grad-CAM \cite{selvaraju2020gradcam} as the explanation method. 
Let $A^k$ denote the activation map of the $k$-th channel in the selected layer for the input $x$. 
Grad-CAM first computes the importance weight $\alpha^k_y$ for each channel using the gradient of the output score $f_y(x)$ with respect to $A^k$, averaged spatially over all locations $(i,j)$, where $Z$ is the total number of spatial positions:
% Assuming $f$ consists of a stack of convolutional layers, Grad-CAM uses the gradients of the predicted probability (for class $y$) with respect to the feature map of selected layer (typically the last layer).
% Given $x$, the activation map at the $k$-th channel is denoted as $A^k$. 
% Grad-CAM computes the importance $\alpha^k_y$ for each channel based on the gradient, where $i,j$ represent the pixel locations, and $Z$ is the map size.
\begin{equation}
    \alpha^k_y = \frac{1}{Z} \sum_{i,j} \frac{\partial f_y (x)}{\partial A^k_{i,j}}.
\end{equation}

The final explanation map $\Phi_y(x)$ is obtained by a weighted combination of the activation maps followed by a ReLU operation to retain only positive influences:
% 
% Next, Grad-CAM multiplies the importance $\alpha^k_y$ by the activation map $A^k$ and sums the results over all channels. 
% A ReLU operation is then applied to obtain $\Phi_y(x)$, which indicates the contribution of each point in $x$ to class $y$:
\begin{equation}
\label{eq:grad_cam}
\Phi_y(x) = \text{ReLU}\left( \sum_k \alpha^k_y A^k \right).    
\end{equation}

\subsection{Explanation-guided learning}
EGL seeks to enhance the model's interpretability by integrating attribution-based signals during training
\cite{ross2017right}. 
% \cite{ross2017right,rieger2020cdep,ismail2021sgt,li2023dre,bertoin2024sgdrop}.
% 
More specifically,  
EGL encourages the alignment of the explanation $\Phi_y(x)$ with a supervision signal, represented by a mask $M(x)$:
\begin{equation}
\label{eq:preliminary_egl}
    \mathcal{L}_{egl} = d\big(\Phi_y(x), M(x)\big),    
\end{equation}
where $d(\cdot,\cdot)$ denotes a divergence measure, such as $L_1$ norm or Binary Cross-entropy (BCE) \cite{ruby2020binary}, 
and the mask $M(x)$ is typically defined by fixed annotations \cite{ross2017right,rieger2020cdep} or predefined segmented results. 
By aligning the explanation with these signals, EGL ensures that the model focuses on relevant regions, improving interpretability.

\section{Methodology}
In this section,
we begin with a \textit{preliminary experiment} (Sec.~\ref{sec:preliminary}) showing that annotation-free masks generated by off-the-shelf segmentation models can be imprecise and may \textit{hinder} model performance.
% in EGL settings. 
We then provide \textit{theoretical analyses} (Sec.~\ref{sec:theoretical}) to underscore the importance of learning high-quality, task-relevant masks to enhance both predictive accuracy and explanation reliability. 
Finally, we introduce \textbf{ALIGN} (Attribution-Learning Iterative Guidance Network) in Sec.~\ref{sec:ALIGN}, a framework that iteratively trains a classifier $f$ and a masker $M$ to focus on semantically meaningful regions of the input.

\subsection{Precise segmentation could improve prediction performance}
\label{sec:preliminary}
As previously discussed, obtaining high-quality human-annotated masks for training is both labor-intensive and costly. 
A common workaround is to leverage large pre-trained segmentation models such as the Segment Anything Model (SAM) \cite{kirillov2023segany} to generate pseudo-masks in a zero-shot manner. 
While these models are capable of producing dense segmentations across diverse inputs, their outputs are not specifically optimized for the downstream prediction task and may include irrelevant or noisy regions. 
Such misaligned or imprecise masks can misguide the learning process, introducing spurious correlations and ultimately degrading predictive performance.

To empirically assess the impact of mask quality, we compare segmentation signals derived from SAM with those generated by our proposed task-driven masker. 
For both types of masks, we apply a controlled background perturbation: the background, those not highlighted by the mask, are blurred using Gaussian noise, while the foreground remains unaltered. 
In this way, the saliency signal identified by the mask is preserved, while background distractions are suppressed. 
Then, well trained classifiers are evaluated on these modified inputs\footnote{
See Sec.~\ref{sec:experiments} for implementation details and Sec.~\ref{sec:ablation_study} for ablation results when \textit{training} with different masks.
}. 
In principle, a high-quality mask should preserve task-relevant features and thus lead to improved performance when the background is suppressed. 
Conversely, poor masks may obscure essential features or retain irrelevant ones, resulting in degraded performance.

\begin{figure}[ht]
  \centering
  \begin{subfigure}{0.495\linewidth}
    \centering
    \includegraphics[width=\linewidth]{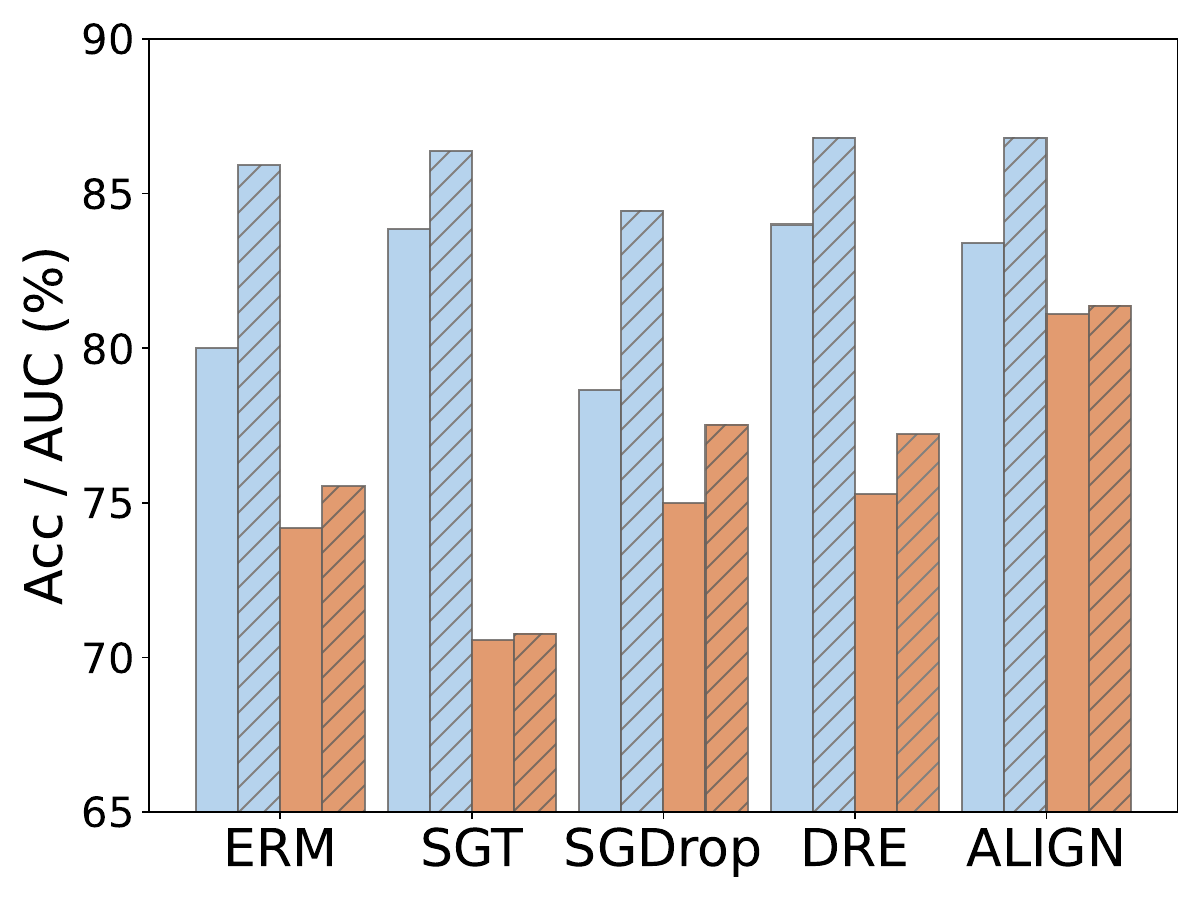}
    \caption{Result on VOC2007}
    \label{fig:bg_VOC2007}
  \end{subfigure}
  \centering
  \begin{subfigure}{0.495\linewidth}
    \centering
    \includegraphics[width=\linewidth]{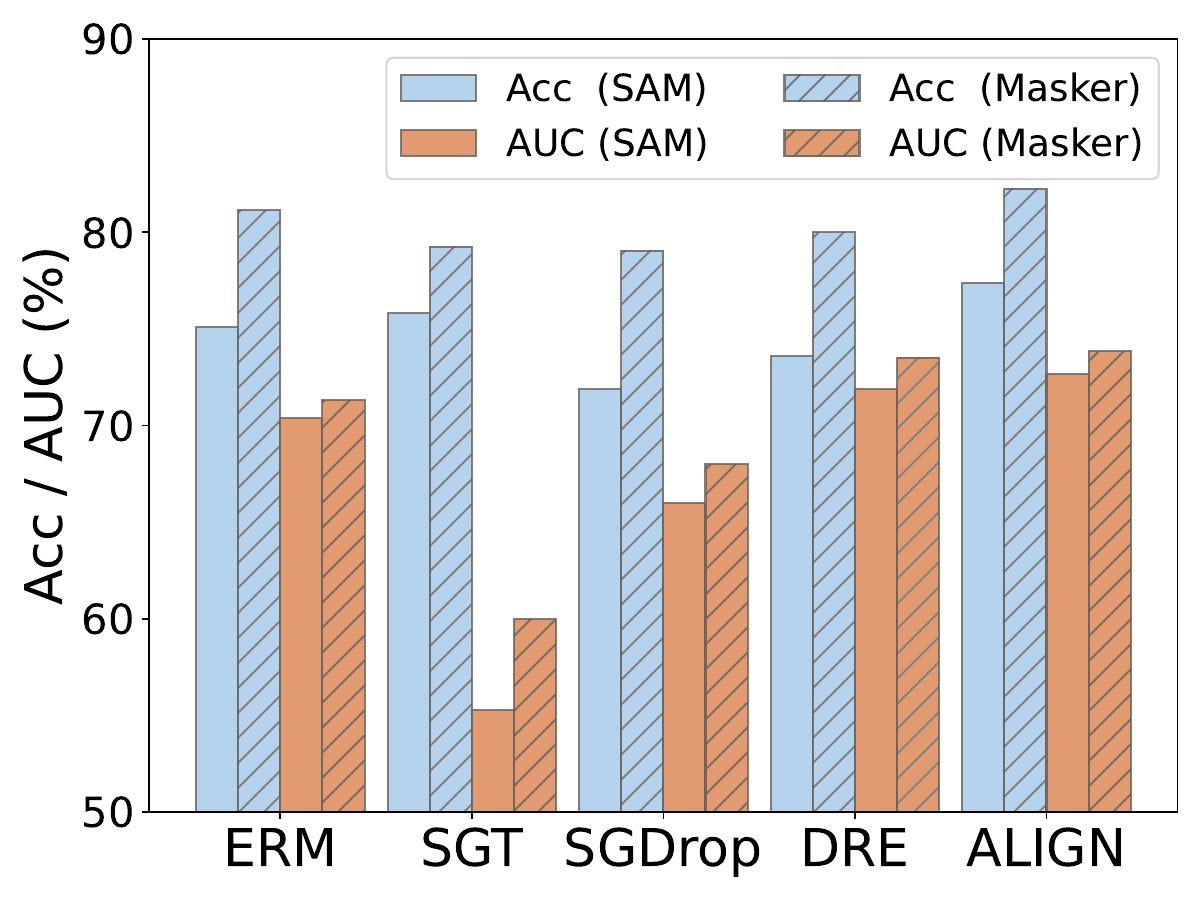}
    \caption{Result on LabelMe}
    \label{fig:bg_LabelMe}
  \end{subfigure}
  \caption{Impact of mask quality for predictions.}
  \label{fig:preliminary}
\end{figure}

Experiments are conducted on VOC2007 and LabelMe, two subsets of the VLCS benchmark \cite{VLCS}. 
We evaluate five models: the standard backbone model (ERM), three state-of-the-art explanation-guided learning methods (SGT, SGDrop, and DRE), and our proposed ALIGN framework. 
As shown in Fig.~\ref{fig:preliminary}, 
all models exhibit noticeable improvements in classification accuracy when using masks produced by our task-driven masker, compared to those generated by SAM\footnote{
The corresponding numerical results can be found in \ref{appendix:mask}.
}.

The results suggest that segmentation quality plays a critical role in guiding model reasoning and prediction. 
The consistent gap between two cases indicates that task-driven masker is more effective than static, pre-trained segmentation outputs. 
It supports our hypothesis: \textit{segmentation signals not optimized for the prediction task can hinder performance}, and precise and task-relevant masks are essential for effective learning. 
Next, we discuss the necessity of high-quality masks formally.

\subsection{Theoretical analysis}
\label{sec:theoretical}

\subsubsection{Basic notations.}
We consider the domain shift scenario, where the input distribution could change between source domain $\mathcal{D}_S$ (training) and a target domain $\mathcal{D}_T$ (testing).
% , such as differences in background textures or lighting, while the fundamental predictive relationship may remain the same. 
% Formally, $\mathcal{D}_S$ and $\mathcal{D}_T$ have marginal distributions $P_S(X)$ and $P_T(X)$ that differ, even if the conditional label relation $P(Y|X)$ remains similar or invariant in certain feature subspaces. 
Assume that each input $x$ can be decomposed into objective $x^{(obj)}$ and background $x^{(bg)}$ according to segmentation $M$: 
\begin{equation}
    x^{(obj)} = M \odot x, \quad 
    x^{(bg)} = (1-M) \odot x.
\end{equation}

We compare three hypotheses (models) correspond to different assumptions of mask $M$.
% belonging to the same hypotheis class $\mathcal{F}$:
\begin{itemize}
    \item $f_1$: Vanilla models can potentially use \textit{all} features in the image, including spurious features. 
    \item $f_2$: Perfect guided models use \textit{all relevant} regions to the prediction, 
    % e.g., $h_2(x)=g \big(M(x)\odot x \big)$ for some $g$ with $\odot$ as the Hadamard product. 
    e.g., $\frac{d f_2(x^{(bg)})}{dx}\approx \mathbf{0}$.
    \item $f_3$: Strict guided models utilize a \textit{strict subset} of $x^{(obj)}$, namely $x^{(sub)}$.
    SAM could lead to a mix of $f_1$ and $f_3$. 
\end{itemize}

\subsubsection{Bounding generalization error under domain shift.}
% To analyze how model predictions vary across domains, 
% consider a parametric path $x(t)$ for $t\in[0,1]$ such that $x(0):=x_S$ is a source sample and $x(1):=x_T$ is a semantically equivalent sample in the target domain (same object, different background). 
% The model output along this path defines a trajectory $f(x(t))$. 
% \begin{equation}
%     \frac{d}{dt} f(x(t)) = \nabla_x f(x(t))^\top \frac{dx(t)}{dt},
% \end{equation}
% where $\nabla_x f$ is the input gradient of the model $f$. 
% % 
% Integrating both sides from $t = 0$ to $t = 1$ gives:
% \begin{equation}
% \label{eq:hxt_hxs_integration_of_path}
%     f(x_T) - f(x_S) = \int_0^1 \nabla_x f(x(t))^\top \frac{dx(t)}{dt} dt \, .
% \end{equation}
% Thus, the total change in prediction as domain shifts equals the integral of the model's input sensitivity (gradient) projected along the path of domain change. 
% % 
% Notice that $f_1$ may have a non-negligible gradient in those spurious directions, 
% % Thus $\nabla_x f_1 \cdot \frac{dx}{dt}$ could be large, 
% leading to a significant output change $f_1(x(1))-f_1(x(0))$.
% Conversely, 
% $f_2$ with $\frac{d f_2(x^{(bg)})}{dx}\approx \mathbf{0}$ naturally results in $\nabla_x f_2 \cdot \frac{dx}{dt} \approx 0$ along any domain shift path, yielding minimal change in output $f_2(x(1)) - f_2(x(0)) \approx 0$. 
% Lemma \ref{lemma:lipschitz_reduction} gives a formal statement:

To analyze how model predictions vary across domains, 
Lemma \ref{lemma:lipschitz_reduction} shows that less invariant features help reduce sensitivity.
In other words, $f_2$ and $f_3$ are expected to be less sensitive than $f_1$ when background changes.

\begin{lemma}
% [Reduced Sensitivity via Less Salient Feature]
\label{lemma:lipschitz_reduction}
Let $x_S, x_T \in \mathbb{R}^d$ be two inputs with identical object features and differing background: $x_S^{(obj)} = x_T^{(obj)}$, $x_S^{(bg)} \neq x_T^{(bg)}$. Define the local Lipschitz constant between $x_S$ and $x_T$ for model $f$ as:
\begin{equation}
    \kappa_f(x_S, x_T) := \frac{|f(x_T) - f(x_S)|}{\|x_T - x_S\|_2}.
\end{equation}
If $f_1$ is sensitive to all features and $f_2$ satisfies $d h_2(x^{(bg)}) / d x \approx \mathbf{0}$, then:
\begin{equation}
\begin{aligned}
    |f_2(x_T) - f_2(x_S)| &< |f_1(x_T) - f_1(x_S)|, \\
    % \qquad \text{and} \qquad
    \kappa_{f_2}(x_S, x_T) &< \kappa_{f_1}(x_S, x_T).
\end{aligned}
\end{equation}
\end{lemma}

% In other words, $h_2$'s predictions remain stable as the input transitions to the new (out-of-distribution) domain.
In practice, adversarially robust models indeed exhibit smaller input-gradients on perceptually irrelevant features \cite{srinivas2023models,tsipras2018robustness}. 
% EGAT explicitly encourages this by highlighting relevant regions, and forcing models to concentrate on these regions of interest. 
% ALIGN explicitly encourages such behavior through saliency supervision.
In other words, explicitly encouraging models to rely on fewer irrelevant features results in more robust behavior. 

Based on the above sensitivity results, 
we further derive discrepancy bounds on the change in mean squared error (MSE) $\Delta_{MSE}$ (Lemma \ref{lemma:mse_discrepancy}) and cross entropy loss $\Delta_{CE}$ (Lemma \ref{lemma:ce_discrepancy}) across domains. 
Both Lemmas imply that $f_2$, which does not rely on spurious features, shows a lower discrepancy bound, and thus better generalizability, than those use all features ($f_1$). 
All the proofs can be found in \ref{appendix:proofs}.

% To analyze generalization under domain shift,
% we introduce a generalization error (risk discrepancy).
% Formally, 
% for a loss function $L(\hat{y},y)$ (e.g. Mean Squared Error or cross-entropy), define the generalization gap under domain shift as
% \begin{equation}
%     \Delta_L(h;\mathcal{D}_S,\mathcal{D}_T)= \big| \mathbb{E}_{(x,y)\sim\mathcal{D}_T} [L(f(x),y)] - \mathbb{E}_{(x,y)\sim\mathcal{D}_S} [L(f(x),y)] \big|.
% \end{equation}

\begin{lemma}[MSE Discrepancy Bound under Domain Shift]
\label{lemma:mse_discrepancy}
% Let $\mathcal{D}_S$ and $\mathcal{D}_T$ be two data distributions over input-label pairs $(x, y)$, such that 
If the conditional distribution $P(y \mid x^{(obj)})$ is the same across domains $\mathcal{D}_S$ and $\mathcal{D}_T$, 
and both the label values $y$ and model prediction $f(x)$ are bounded: $|y| \leq 1$ and $|f(x)|\leq 1$. Then, the MSE discrepancy under domain shifts:
\begin{equation}
\begin{aligned}
\label{eq:mse_discrepancy_bound}
\Delta_{MSE} & (f; \mathcal{D}_S,\mathcal{D}_T) 
:= \left| \mathbb{E}_{\mathcal{D}_T}[(f(x) - y)^2] - \mathbb{E}_{\mathcal{D}_S}[(f(x) - y)^2] \right| \\
& \leq 4 \left| \mathbb{E}_{\mathcal{D}_T}[f(x)] - \mathbb{E}_{\mathcal{D}_S}[f(x)] \right| 
+ \left| \mathbb{E}_{\mathcal{D}_T}[y^2] - \mathbb{E}_{\mathcal{D}_S}[y^2] \right|.
\end{aligned}
\end{equation}
% If in addition 
When 
$\mathbb{E}_{\mathcal{D}_T}[y^2] = \mathbb{E}_{\mathcal{D}_S}[y^2]$ (i.e., no label distribution shift):
% , then:
\begin{equation}
\Delta_{MSE}
% =\left| \mathbb{E}_{\mathcal{D}_T}[(f(x) - y)^2] - \mathbb{E}_{\mathcal{D}_S}[(f(x) - y)^2] \right|
\leq 4 \left| \mathbb{E}_{\mathcal{D}_T}[f(x)] - \mathbb{E}_{\mathcal{D}_S}[f(x)] \right|.
\end{equation}
\end{lemma}

\begin{lemma}[Cross-Entropy Stability to Small Shifts]
\label{lemma:ce_discrepancy}
Suppose for every input $x$ and its label $y$, the difference in predicted probabilities between the two domains is bounded: 
% $\big|h_{\mathcal{D}_T}(x)_y - h_{\mathcal{D}_S}(x)_y\big| \leq \epsilon$, where $h_{\mathcal{D}_S}(x)_y$ denotes the probability assigned to the true class by the classifier when encountering $x$ in the source domain distribution. 
$\big| f(x_T)-f(x_S) \big| \leq \epsilon$,
then the absolute difference in cross-entropy risk is bounded as 
\begin{equation}
    \Delta_{CE}:=|CE_{\mathcal{D}_T}(f) - CE_{\mathcal{D}_S}(f)| \leq C \cdot \epsilon,
\end{equation}
for some constant $C$ that depends only on the range of $f(x)$, e.g., $C = \frac{1}{\min_i {f(x)_i}}$. 
% In simpler terms, small shifts in the model’s output distribution imply small shifts in the cross-entropy loss.
\end{lemma}

% \begin{proof}[Proof Sketch]
% The cross-entropy loss for a single example can be written as $CE(f(x),y) = -\log f(x)_y$. 
% The function $q \mapsto -\log q$ is continuously differentiable and its derivative is bounded on $(0,M)$ if $q$ is bounded away from 0 by some $1/M$ ($M\ll 1)$. 
% In other words, $-\log f(x)_y$ is Lipschitz continuous on the range of $f(x)_y$, which implies that a change of $\epsilon$ in $f(x)_y$ causes at most $M\epsilon$ change in $-\log f(x)_y$. 
% % In our case, $q = f(x)_y$ changes by at most $\epsilon$ between domains for each $(x,y)$ by assumption. Therefore $|\ell_S(f(x),y)-\ell_T(f(x),y)| \leq M\epsilon$ pointwise. 
% Taking expectation over $(x,y)$ yields the bound. 
% % For reasonably well-calibrated classifiers $M$ is not large; if not, one can restrict to a probability range $[\delta,1-\delta]$ to avoid the singularity at 0.
% \end{proof}

% By considering Lemma \ref{lemma:lipschitz_reduction} and \ref{lemma:ce_discrepancy} together, 
% it shows that if the classifier's output probabilities do not drastically change, e.g., $h_2$, its cross-entropy loss across domains $\Delta_{CE}(h_2)$ does not drastically change as well. 

\begin{figure*}
    \centering
    \includegraphics[width=0.97\linewidth]{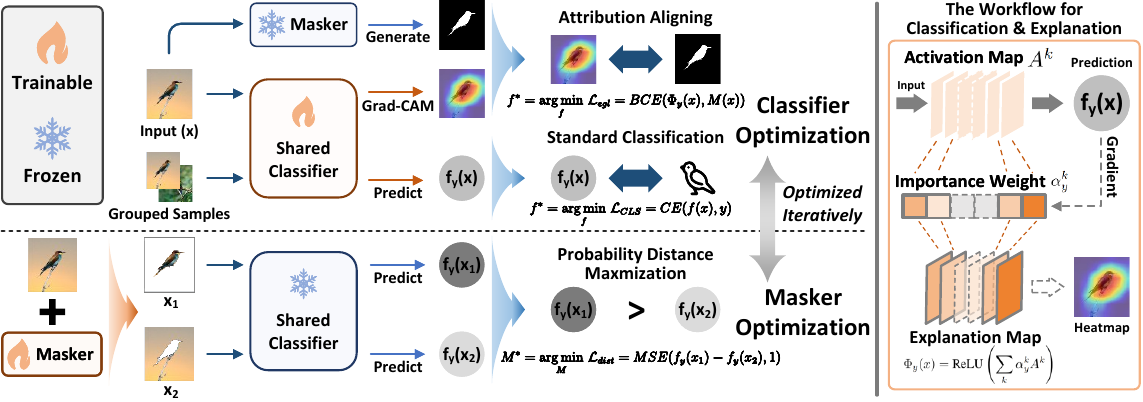}
    \caption{Overview of the proposed ALIGN framework.}
    \label{fig:framework}
\end{figure*}

Note that $f_3$ outperforms $f_1$ by strictly restricting its input to a subset of the salient features $x^{(obj)}$, as supported by the preceding lemmas. However, as established in Lemma \ref{lemma:in_domain_error_feature_inclusion}, $f_2$ shows superiority over $f_3$ concerning in-domain performance, achieving both lower mean squared error and cross-entropy loss by leveraging the complete set of relevant features.
These analyses motivate the design of \textit{a learnable masker to guide model training} toward more effective feature selection. 

\begin{lemma}[In-domain errors concerning Feature Inclusion]
\label{lemma:in_domain_error_feature_inclusion}
Let $f_2^*(x^{{(obj)}}) := \mathbb{E}[y \mid x^{(obj)}]$ and $f_3^*(x^{(sub)}) := \mathbb{E}[y \mid x^{(sub)}]$ denote the Bayes optimal predictors using feature $x^{(obj)}$ and $(x^{(sub)})$, respectively. 
Then, the associated Bayes risks under squared loss satisfy:
\begin{equation}
\mathbb{E}_{\mathcal{D}_S}\left[(y - f_2^*(x^{(obj)}))^2\right] \leq \mathbb{E}_{\mathcal{D}_S}\left[(y - f_3^*(x^{(sub)}))^2\right],
\end{equation}
and the following cross-entropy error inequality holds:
\begin{equation}
    \mathbb{E}_{\mathcal{D}_S}[-\log f_2(x)_y] \leq \mathbb{E}_{\mathcal{D}_S}[-\log f_3(x)_y].
\end{equation}
Both of which with strict inequality if there exists a feature $x_j \in x^{(obj)} \setminus x^{(sub)}$ such that $x_j \not\!\perp\!\!\!\perp y \mid x^{(sub)}$.
\end{lemma}

\subsection{The ALIGN framework}
\label{sec:ALIGN}
Motivated by our empirical observations and theoretical findings, we argue that learning a task-driven, high-quality masker during training is both necessary and beneficial for enhancing model interpretability and generalization. 
To this end, we propose \textbf{Attribution-Learning Iterative Guidance Network (ALIGN)}.
As illustrated in Fig.~\ref{fig:framework}, ALIGN jointly learns a masker $M$ and a classifier $f$ in an iterative manner.
The classifier $f$ is instantiated as a standard ResNet \cite{he2016resnet}, while the masker $M$ is a lightweight convolutional network that produces a soft mask $M(x) \in [0,1]^d$, highlighting task-relevant regions of the input. 

ALIGN consists of two interleaved optimization steps: one for refining the \textit{masker} to generate smooth and semantically meaningful regions that preserve prediction-relevant information, and the other for updating the \textit{classifier} based on both prediction and attribution alignment objectives.

\paragraph{Masker objective.}
Predefined segmentation models often produce imprecise or task-irrelevant masks, introducing noise that can hinder prediction performance. 
Thus, we propose learning a task-driven masker $M$ that automatically identifies input regions relevant to the model's prediction. 
The core goal of $M$ is to highlight informative features while suppressing misleading background regions.

Formally, we encourage the model to predict high confidence on the masked input $x \odot M(x)$ (i.e., retained foreground), and low confidence on the complementary region $x \odot (1 - M(x))$ (i.e., suppressed background). 
This contrast is captured by the probability distance:
\begin{equation}
    dist(x) = f_y(x \odot M(x)) - f_y(x \odot (1-M(x))) ~,
\end{equation}
where $f_y(\cdot)$ denotes the predicted probability for the ground-truth class $y$. 
The main objective for the masker is:
\begin{equation}
    \mathcal{L}_{dist} = MSE(dist(x),1) ~,
\end{equation}
which encourages $f$ to rely more on the foreground regions identified by $M$. Since $dist(x) \in (-1, 1)$, the mean squared error is used to enforce a high contrast.

To ensure that the learned masks are both interpretable and effective, we incorporate two regularization terms. 
The sparsity loss penalizes unnecessary activations, encouraging compact masks:
\begin{equation}
    \mathcal{L}_{sparsity} = \|M(x)\|_1.
\end{equation}

The smoothness loss enforces spatial continuity by penalizing abrupt changes between neighboring pixels:
\begin{equation}
    \begin{aligned}
    \mathcal{L}_{smooth} = \frac{1}{Z} \sum_{i,j} 
     ( & \left| M_{i,j}(x) - M_{i+1,j}(x) \right| \\
     & + \left| M_{i,j}(x) - M_{i,j+1}(x) \right| ).
    \end{aligned}
\end{equation}

% Additionally, the regularization loss terms $\mathcal{L}_{sparsity}$ and $\mathcal{L}_{smooth}$ are introduced to enforce desirable properties on the mask $M(x)$. 
% Specifically, $\mathcal{L}_{sparsity}$ promotes sparsity by encouraging the mask to have fewer active regions, which helps in focusing the model on the most critical parts of the input. On the other hand, $\mathcal{L}_{smooth}$ enforces smooth transitions between neighboring mask values, ensuring that the mask does not have abrupt changes between adjacent areas.

% \begin{equation}
%     \mathcal{L}_{sparsity} = \|M(x)\|_1 ~,
% \end{equation}
% \begin{equation}
%     \begin{aligned}
%     \mathcal{L}_{smooth} = \frac{1}{Z} \sum_{i,j} 
%      ( & \left| M_{i,j}(x) - M_{i+1,j}(x) \right| \\
%      & + \left| M_{i,j}(x) - M_{i,j+1}(x) \right| ) ~.
%     \end{aligned}
% \end{equation}

In summary, the overall loss for the masker $M$ is:
\begin{equation}
\mathcal{L}_{mask} = \mathcal{L}_{dist} 
+ \lambda_{1} \mathcal{L}_{sparsity} 
+ \lambda_{2} \mathcal{L}_{smooth},
\end{equation}
% where the parameters of classifier $f$ are frozen, and only the masker $M$ is trained.
where $\lambda_1>0$ and $\lambda_2>0$ are hyper-parameters.

\paragraph{Classifier objective.}
The classifier $f$ is trained to minimize a composite objective:
\begin{equation}
\label{eq:L_clf_total}
\mathcal{L}_{clf} = \mathcal{L}_{cls} + \lambda_{3} \mathcal{L}_{egl} + \lambda_{4} \mathcal{L}_{reg},
\end{equation}
comprising the following components:
(1) Classification loss $\mathcal{L}_{cls} = \text{CE}(f(x), y)$ is the standard cross-entropy loss used to ensure accurate predictions.
(2) Explanation guided loss $\mathcal{L}_{egl} = \text{BCE}(\Phi_y(x), M(x))$ aligns the classifier’s explanation $\Phi_y(x)$ 
% (e.g., Grad-CAM heatmap) 
with the mask $M(x)$ generated by the masker, using binary cross-entropy.
(3) Mixup-based regularization: Inspired by \cite{li2023dre}, we use a mixup strategy for both input and explanation.
Formally, 
for two inputs $(x_i, y)$ and $(x_j, y)$ with the same class, a synthetic sample is constructed as:
\begin{equation}
        \tilde{x} = \beta x_i + (1 - \beta) x_j, \quad \beta \sim \text{Beta}(\alpha, \alpha).
\end{equation}
The regularization loss is defined as:
\begin{equation}
\begin{aligned}
    \mathcal{L}_{reg} = & \| \beta \Phi(x_i) + (1 - \beta) \Phi(x_j) - \Phi(\tilde{x}) \|_1 \\
    & + \text{CE}(f(\tilde{x}), y) + \|\Phi(\tilde{x})\|_1,
\end{aligned}
\end{equation}
which encourages consistency in both prediction and attribution space, while promoting explanation sparsity.

\paragraph{Joint optimization.}
As illustrated in Fig.~\ref{fig:framework}, the ALIGN framework is trained in an alternating optimization scheme. 
Specifically, during each training iteration, the masker $M$ is updated while keeping the classifier $f$ fixed. 
Subsequently, the classifier $f$ is updated using the latest output from $M$, with $M$ parameters held constant.

To mitigate cold-start issues and enhance training stability in the early stages, the process begins with classifier-only training.
At this stage, only the standard classification loss $\mathcal{L}_{cls}$ and the mix-up based regularization $\mathcal{L}_{reg}$ are applied, and no explanation supervision is imposed (i.e., $\mathcal{L}_{egl} = 0$).
This allows the classifier to form a reliable initial decision without being influenced by potentially unstable masks.

The masker is introduced after a predefined warm-up phase (set to 200 in our paper). 
The iterative training then continues with both $\mathcal{L}_{mask}$ and $\mathcal{L}_{clf}$ active, gradually aligning the model’s reasoning with the learned masks.

\begin{table*}[h]
    \footnotesize
    \setlength{\tabcolsep}{1.5mm}
    \centering
    \begin{tabular}{c|cccc|cccc|cccc|cccc}
    \toprule
        \multirow{2}{*}{\textbf{Method}} & \multicolumn{4}{c|}{\textbf{VOC2007}} & \multicolumn{4}{c|}{\textbf{LabelMe}} & \multicolumn{4}{c|}{\textbf{Caltech101}} & \multicolumn{4}{c}{\textbf{SUN09}} \\ 
        \cmidrule{2-17}
        % \cmidrule(lr){2-5} \cmidrule(lr){6-9} \cmidrule(lr){10-13} \cmidrule(lr){14-17}
        % \cmidrule{2-17}
        & Acc & AUC & Suff & Comp & Acc & AUC & Suff & Comp & Acc & AUC & Suff & Comp & Acc & AUC & Suff & Comp \\ \midrule 
        ERM & 85.35 & 76.95 & 17.18 & 15.64 & \textbf{80.80} & 71.08 & 1.22 & 15.21 & 99.73 & 96.86 & 33.71 & 2.22 & 80.87 & \textbf{77.06} & 29.42 & 11.20 \\ 
        % IRM & 85.22 & \textbf{82.34} & 38.77 & 9.05 & 76.41 & 71.43 & 22.14 & 7.16 & 95.35 & 98.01 & 30.31 & 2.65 & 80.46 & 74.05 & 37.43 & 12.23 \\
        IRM & 81.45 & 75.46 & 26.38 & 11.83 & 77.47 & 68.36 & 8.49 & 14.10 & 98.86 & 94.71 & 11.41 & \textbf{10.34} & 81.45 & 75.46 & 26.38 & 11.83 \\ 
        Mixup & 85.55 & 74.38 & \textbf{10.87} & \textbf{24.47} & 79.73 & 69.31 & 23.21 & \textbf{17.14} & 99.58 & 94.84  & 30.26  & 3.56 & 78.85 & 72.57 & \textbf{17.75} & 9.34 \\ 
        SGT & 86.64 & 72.91 & 17.72 & 18.55 & 79.54 & 60.80 & 26.43 & 12.85 & 99.52 & 95.43 & 22.39 & 2.28 & 79.23 & 63.20 & 32.87 & 12.36 \\ 
        SGDrop & 86.26 & 78.37 & 18.02 & 17.35 & 79.25 & 68.81 & 12.89 & 13.29 & 99.94 & 97.70 & 12.30 & 1.68 & 80.15 & 75.54 & 23.98 & 8.60 \\ 
        DRE & 85.61 & 77.41 & 16.54 & 17.02 & 80.31 & 73.77 & 10.48 & 15.67 & 99.95 & \textbf{99.34} & 8.61 & 1.50 & 81.76 & 70.96 & 34.62 & 8.88 \\ \midrule
        ALIGN & \textbf{86.91} & \textbf{82.18} & 15.08 & 16.66 & 80.23 & \textbf{74.29} & \textbf{1.00} & 13.32 & \textbf{99.98} & 99.05 & \textbf{4.71} & 2.81 & \textbf{82.54} & 71.16 & 20.75 & \textbf{12.54} \\ 
        
        % \bottomrule 
        % \toprule
        \midrule
        \midrule
        
        \multirow{2}{*}{\textbf{Method}} & \multicolumn{4}{c|}{\textbf{Loc\_38}} & \multicolumn{4}{c|}{\textbf{Loc\_43}} & \multicolumn{4}{c|}{\textbf{Loc\_46}} & \multicolumn{4}{c}{\textbf{Loc\_100}} \\ 
        \cmidrule{2-17}
        % \cmidrule{2-17}
        & Acc & AUC & Suff & Comp & Acc & AUC & Suff & Comp & Acc & AUC & Suff & Comp & Acc & AUC & Suff & Comp \\ \midrule
        ERM & 77.89 & 61.03 & \textbf{13.78} & 35.59 & \textbf{76.35} & 56.25 & 13.93 & 43.62 & 72.69 & 58.95 & 14.03 & 35.74 & 88.47 & 77.88 & 19.12 & 36.92 \\ 
        % IRM & 65.77 & 57.89 & 34.86 & 30.04 & 52.29 & \textbf{67.42} & 36.18 & 19.98 & 60.08 & 68.59 & 35.05 & 20.63 & 83.93 & 78.43 & 35.85 & 7.61 \\
        IRM & 80.50 & 59.47 & 11.58 & 36.47 & 74.05 & 63.31 & 9.95 & 49.15 & 72.69 & 59.46 & 21.32 & 38.18 & 87.53 & 71.63 & 13.38 & 37.58 \\ 
        Mixup & 77.41 & 57.25 & 14.94 & 34.14 & 67.28 & 55.74 & 19.65 & 38.86 & 73.85 & 58.78  & 11.52  & 36.63 & 88.72 & 75.89 & 16.48 & 38.79 \\
        SGT & 77.87 & 51.27 & 13.97 & 29.51 & 74.73 & 60.75 & 12.07 & 39.43 & 71.59 & 58.67 & 17.37 & 35.36 & 86.41 & 74.87 & 16.38 & 29.64 \\ 
        SGDrop & 80.26 & 63.85 & 21.68 & 46.41 & 73.38 & 56.72 & 29.82 & 60.21 & 74.94 & 67.31 & 14.73 & 35.07 & 89.43 & 79.15 & \textbf{9.43} & 36.08 \\ 
        DRE & 77.37 & 65.02 & 17.58 & 31.21 & 74.89 & 62.64 & 17.65 & \textbf{60.60} & 73.95 & 65.73 & \textbf{10.87} & 35.30 & 88.39 & 80.15 & 14.63 & 31.56 \\ \midrule
        ALIGN & \textbf{83.62} & \textbf{66.83} & 19.75 & \textbf{51.66} & 72.47 & \textbf{65.05} & \textbf{8.20} & 44.33 & \textbf{77.27} & \textbf{69.83} & 13.94 & \textbf{40.54} & \textbf{90.54} & \textbf{84.13} & 16.28 & \textbf{42.88} \\ \bottomrule
    \end{tabular}
    \caption{Overvall performance comparison of ALIGN and baselines arcoss eight sub-datasets. 
    \textbf{Bold} values indicate the winner for each metric and dataset.
    ALIGN consistently achieves the best or competitive performance across all metrics.} 
    \label{tab:main}
\end{table*}

% ======================================
% EGL 
% ======================================

\section{Experiments}
\label{sec:experiments}
\subsection{Datasets and settings}
We evaluate the baselines on VLCS and Terra Incognita datasets. VLCS~\cite{VLCS} includes four domains (VOC2007, LabelMe, Caltech101, SUN09) with about 25K images in 5 classes, while Terra Incognita~\cite{xu2020adversarialmixup} comprises four locations (38, 43, 46, 100) totaling around 11K images in 10 categories.

% We conduct evaluations of baseline models on the VLCS and Terra Incognita datasets. All images are resized to $224\times224$.
% \begin{itemize}
%     \item \textbf{VLCS} \cite{VLCS} consists of four sub-datasets: VOC2007, LabelMe, Caltech101, and SUN09. 
%     It comprises about 25,000 images in 5 classes.
    
%     \item \textbf{Terra Incognita} \cite{xu2020adversarialmixup} contains four sub-datasets (locations 38, 43, 46, and 100),
%     including about 11,000 images across 10 categories.
% \end{itemize}

All sub-datasets were split into training, validation, and test sets with a ratio of 6:2:2. 
The detailed setting of environment and hyper-parameters can be found in \ref{appendix:config}.

\subsection{Baselines \& Evaluation Metrics}
We compare ALIGN with several representative baselines, including ERM~\cite{he2016resnet}, IRM~\cite{arjovsky2019irm}, Mixup~\cite{xu2020adversarialmixup}, SGT~\cite{ismail2021sgt}, SGDrop~\cite{bertoin2024sgdrop}, and DRE~\cite{li2023dre}.

% To comprehensively evaluate the effectiveness of ALIGN, 
% We compare ALIGN against the following baselines:
% \begin{itemize}
%     \item ERM~\cite{he2016resnet} minimizes loss without explicitly considering domain shifts.
%     \item IRM \cite{arjovsky2019irm} learns invariant predictors via invariance risk minimization.
%     \item Mixup~\cite{xu2020adversarialmixup} enhances robustness through linear interpolation of training samples.
%     \item SGT~\cite{ismail2021sgt} uses saliency to guide attention toward informative regions.
%     \item SGDrop~\cite{bertoin2024sgdrop} 
%     selectively suppresses dominant feature activations to improve generalizability.
%     \item DRE~\cite{li2023dre} models uncertainty to improve reliability under domain shifts.
% \end{itemize}

% \subsection{Evaluation Metrics}
In addition to prediction metrics such as \textbf{AUC} and \textbf{Acc}, 
the explanations are evaluated by two metrics: \textbf{Sufficiency (Suff)} and \textbf{Comprehensiveness (Comp)} \cite{guidedlearning}.

% \begin{itemize}

%     \item \textbf{Sufficiency (Suff)} quantifies the degradation in model performance when only the salient regions of the input are retained. It is formally defined as:
%     \begin{equation}
%         Suff = f_y(x) - f_y(g(x, \Phi_y(x))),
%     \end{equation}
%     where $g(x, \Phi_y(x))$ represents the subset of input pixels deemed highly important by the explanation. 
%     A \textit{lower} Suff value indicates that the explanation captures more essential information needed for accurate prediction.

%     \item \textbf{Comprehensiveness (Comp)} 
%     % assesses how critical the identified salient regions are to model’s decision. 
%     % It 
%     measures the drop in prediction confidence when the highlighted regions are removed. Formally, 
%     \begin{equation}
%         Comp = f_y(x) - f_y(x \setminus g(x, \Phi_y(x))),
%     \end{equation}
%     where a \textit{high} value indicates that the model strongly relies on the identified regions to make its prediction.
% \end{itemize}

\subsection{Main Results}
% To quantitatively evaluate the performance of ALIGN, 
% we conducted a series of comparative experiments using the VLCS and Terra Incognita datasets, benchmarking against several well-established baseline models, including both conventional and explanation-aware approaches. 
% The main experimental results are presented in Table~\ref{tab:main}.

To quantitatively assess \textbf{ALIGN}, we conduct experiments on VLCS and Terra Incognita, benchmarking against both standard and explanation-aware baselines. Table~\ref{tab:main} summarizes the results for prediction and explanation metrics.

% In the majority of cases, ALIGN outperforms the baseline models, excelling not only in traditional prediction metrics such as Acc and AUC, but also surpassing them in interpretability metrics, including Suff and Comp. 

On VLCS, ALIGN achieves the best accuracy and AUC in most domains (VOC2007, CAltech101, SUN09), consistently outperforming all baselines.
In terms of explanation quality, ALIGN yields competitive or superior results across both Suff and Comp metrics, indicating that its predictions rely on concise and interpretable evidence.

Similarly, in Terra Incognita dataset, ALIGN again leads in Acc and AUC for most locations (38, 43, 100) and consistently excels in interpretability metrics, demonstrating strong generalization in complex, real-world settings.

These improvements can be attributed to guidances provided by the masker in ALIGN, which effectively identifies task-relevant regions, enabling the classifier to focus on the most relevant input areas. 
To further assess its effect, we compare alternative masking strategies, including pretrained segmentation and heuristic methods (Sec.~\ref{sec:ablation_study}).

In a few cases, ALIGN may not achieve the absolute best performance but remains highly competitive. 
As noted in Lemma~\ref{lemma:in_domain_error_feature_inclusion}, this can occur when the generated mask inadvertently omits a few relevant features, causing slight degradation in in-distribution accuracy.
Nonetheless, as shown in Sec.~\ref{sec:experiment_ood}, ALIGN shows clear advantages under domain shifts by reducing dependence on spurious background cues.

% These improvements can be attributed to the guidance of masker integrated into ALIGN, which effectively identifies and highlights task-specific regions, enabling the classifier to focus on the most relevant areas.
% To deeply explore the benefit of the mask, we further study the variants when adopting other mask generation strategies, such as pretrained segmentation models, or a heurstrical method in Sec.~\ref{sec:ablation_study}.
% Furthermore, by incorporating this approach within EGL during classifier training, ALIGN not only improves its predictive accuracy but also enhances the transparency and interpretability of its outputs.
% This dual benefit makes ALIGN a promising solution for both accurate and explainable models.

\begin{table*}[h]
    \footnotesize
    \setlength{\tabcolsep}{1.5mm}
    \centering
    \begin{tabular}{l|cccc|cccc|cccc|cccc}
    \toprule
        \multirow{2}{*}{\textbf{Method}} & \multicolumn{4}{c|}{\textbf{VOC2007}} & \multicolumn{4}{c|}{\textbf{LabelMe}} & \multicolumn{4}{c|}{\textbf{Caltech101}} & \multicolumn{4}{c}{\textbf{SUN09}} \\ 
        % \cline{2-17}
        \cmidrule{2-17}
        & Acc & AUC & Suff & Comp & Acc & AUC & Suff & Comp & Acc & AUC & Suff & Comp & Acc & AUC & Suff & Comp \\ \midrule 
        w/o EG & 85.61 & 77.41 & 16.54 & 17.02 & \textbf{80.31} & 73.77 & 10.48 & \textbf{15.67} & 99.95 & \textbf{99.34} & 8.61 & 1.50 & 81.76 & 70.96 & 34.62 & 8.88 \\ 
        m-SAM & 85.51 & 80.32 & \textbf{13.11} & \textbf{18.43} & 77.66 & 70.64 & 10.25 & 7.14 & 99.80 & 98.43 & 5.63 & 2.12 & 79.70 & 68.51 & 26.42 & 8.74 \\ 
        m-Gray & 86.90 & 79.15 & 13.45 & 14.31 & 78.79 & 69.83 & 10.09 & 10.30 & 99.86 & 97.42 & 11.97 & 1.21 & 80.13 & 70.59 & 32.75 & 7.53 \\ \midrule
        ALIGN & \textbf{86.91} & \textbf{82.18} & 15.08 & 16.66 & 80.23 & \textbf{74.29} & \textbf{1.00} & 13.32 & \textbf{99.98} & 99.05 & \textbf{4.71} & \textbf{2.81} & \textbf{82.54} & \textbf{71.16} & \textbf{20.75} & \textbf{12.54} \\ 
        
        \midrule
        \midrule
        % \bottomrule 
        % \toprule

        \multirow{2}{*}{\textbf{Method}} & \multicolumn{4}{c|}{\textbf{Loc\_38}} & \multicolumn{4}{c|}{\textbf{Loc\_43}} & \multicolumn{4}{c|}{\textbf{Loc\_46}} & \multicolumn{4}{c}{\textbf{Loc\_100}} \\ 
        % \cline{2-17}
        \cmidrule{2-17}
        & Acc & AUC & Suff & Comp & Acc & AUC & Suff & Comp & Acc & AUC & Suff & Comp & Acc & AUC & Suff & Comp \\ \midrule
        
        w/o EG & 77.37  & 65.02  & 17.58 & 31.21  & \textbf{74.89}  & 62.64  & 17.65  & \textbf{60.60}  & 73.95  & 65.73  & \textbf{10.87}  & 35.30  & 88.39  & 80.15 & 14.63 & 31.56 \\
        m-SAM & 80.02  & \textbf{69.72} & \textbf{7.57} & 30.10  & 72.02 & 63.07 & 13.86 & 44.13 & 73.06 & 67.12 & 23.09 & 40.48 & 89.13 & 79.51 & \textbf{11.78} & 38.54 \\ 
        m-Gray & 78.04 & 58.34 & 11.03 & 30.20  & 70.09 & 55.72 & 11.05 & 39.75 & 75.92 & \textbf{70.16} & 17.20  & 35.42 & 89.45 & 82.19 & 21.09 & 38.39 \\ \midrule
        ALIGN & \textbf{83.62}  & 66.83 & 19.75 & \textbf{51.66}  & 72.47  & \textbf{65.05}  & \textbf{8.20}  & 44.33 & \textbf{77.27}  & 69.83 & 13.94 & \textbf{40.54}  & \textbf{90.54}  & \textbf{84.13}  & 16.28 & \textbf{42.88} \\ \bottomrule

    \end{tabular}
    \caption{Ablation study on the role of the masker in ALIGN.
    The winners are in \textbf{bold}.}
    \label{tab:abla}
\end{table*}

\subsection{Ablation Study: Effectiveness of the Masker}
\label{sec:ablation_study}
In Sec.~\ref{sec:preliminary}, we found that ``precise segmentation can better improve prediction performance'', showing that masks generated by our proposed masker better support model \textit{inference} than those produced by SAM. 
Here, we further investigate the impact of the masker during \textit{training} by comparing ALIGN against three alternative variants. Each variant modifies the mask generation strategy to assess the contribution of the explanation-guided loss and mask quality:
\begin{itemize}
    \item \textbf{w/o EG} disables the EGL component during training by setting $\lambda_3=0$ in Eq. (\ref{eq:L_clf_total}). 
    % It evaluates model performance \textit{without} explanation guidance.
    
    \item \textbf{m-SAM} replaces the learned mask $M(x)$ with segmentation maps generated by the pretrained SAM \cite{kirillov2023segany}, as fixed targets for saliency alignment.

    \item \textbf{m-Gray} uses grayscale intensity values as the mask, i.e., $M(x) = \text{Gray}(x)$, encouraging the classifier’s saliency map to align with low-level pixel brightness. 
\end{itemize}

From the results reported in Table~\ref{tab:abla}, we observe that:
(1) Incorporating external mask signals for EGL (m-SAM, m-Gray) improves performance over the variant without EGL (w/o EG), confirming the value of explanation-based supervision.
(2) ALIGN, which dynamically generates task-relevant masks, further outperforms all fixed alternatives, demonstrating the advantage of task-driven, end-to-end learning of explanation signals.

\subsection{Out-of-Distribution Generalization}
\label{sec:experiment_ood}
To complement the theoretical analysis presented in Sec. \ref{sec:theoretical} and further assess the generalizability of ALIGN, we conduct out-of-distribution (OOD) experiments. 
In this setting, the model is trained on a source domain and evaluated on remaining target domains \textit{without retraining}. 
This setup simulates real-world distribution shifts where models must generalize to unseen environments.

\begin{table}[h]
    \footnotesize
    \setlength{\tabcolsep}{0.8mm}
    \centering
    \begin{tabular}{c|c|ccccc}
    \toprule
        \textbf{Train on} & \textbf{Test on} & \textbf{ERM} & \textbf{SGT} & \textbf{SGDrop} & \textbf{DRE} & \textbf{ALIGN} \\
    \midrule
        
        \multirow{3}{*}{VOC2007} 
        & LabelMe & 59.43 & \textbf{62.12} & 59.81 & 51.71 & 56.89 \\ 
        & Caltech101 & 98.31 & 99.48 & 98.73 & 99.52 & \textbf{99.53} \\ 
        & SUN09 & 73.02 & 69.23 & 71.91 & 73.74 & \textbf{74.03} \\ 
    \midrule
        \multirow{3}{*}{LabelMe} 
        & VOC2007 & 67.27 & 70.12 & 61.39 & 58.53 & \textbf{73.63} \\ 
        & Caltech101 & 90.04  & 91.92 & 89.92 & 89.59 & \textbf{96.63} \\ 
        & SUN09 & 56.36 & 59.58 & 50.03 & 51.35 & \textbf{61.36} \\ 
    \midrule
        \multirow{3}{*}{Caltech101} 
        & VOC2007 & 51.68 & 41.53 & 42.10 & \textbf{53.11} & 49.60 \\ 
        & LabelMe & 38.89 & 33.68 & 33.32 & 41.64 & \textbf{42.24} \\ 
        & SUN09 & 42.03 & 34.21 & 37.61 & 43.87 & \textbf{48.04} \\ 
    \midrule
        \multirow{3}{*}{SUN09} 
        & VOC2007 & 64.65 & 66.56 & 63.51 & \textbf{66.76} & 61.01 \\ 
        & LabelMe & 59.66 & 57.86 & 52.82 & 57.78 & \textbf{62.07} \\ 
        & Caltech101 & 73.37 & 73.97 & 65.08 & \textbf{83.24} & 65.77 \\ 
    \bottomrule
    \end{tabular}
    \caption{Prediction accuracy on VLCS under OOD setting.}
    \label{tab:ood}
\end{table}

% Table~\ref{tab:ood} reports accuracy across all source-target domain pairs in VLCS. 
% More comprehensive OOD results, including AUC scores and Terra Incognita experiments, can be found in Appendix~\ref{appendix:ood}.

Table~\ref{tab:ood} reports accuracy across all source-target domain pairs in VLCS. \footnote{More comprehensive results can be found in Appendix~\ref{appendix:ood}.}
Overall, ALIGN consistently outperforms all baselines across most OOD settings, often by a substantial margin. 
For instance, when trained on the LabelMe and tested on Caltech101, ALIGN achieves an impressive 96.63\% accuracy, significantly surpassing the strongest baseline, SGT, which achieves 91.92\%. 

We attribute the improved generalization to the trainable masker, which aligns classifier saliency with learned task-relevant regions. This alignment helps identify domain-invariant features, reducing overfitting to spurious source correlations and enhancing performance on unseen domains.

\subsection{Case Study}

\paragraph{Post-hoc Explanation.}
To better understand how ALIGN improves both prediction accuracy and interpretability, 
we conduct a case study comparing the attribution patterns of models. 
Fig.~\ref{fig:case_study} visualizes the Grad-CAM saliency maps for two samples from the VLCS dataset.
The ground truth is shown in brackets, with the corresponding predictions from each model are reported below each heatmap.
% , along with model confidence scores for the ground truth class.
% 
% Specifically, we visualize the attribution heatmaps produced by ERM, Mixup, SGT, DRE, and ALIGN to analyze their reasoning processes in detail.
% The results are reported in Figure~\ref{fig:case_study}, illustrating notable differences in attention across models. 

\begin{figure}[h]
    \centering
    \includegraphics[width=0.95\linewidth]{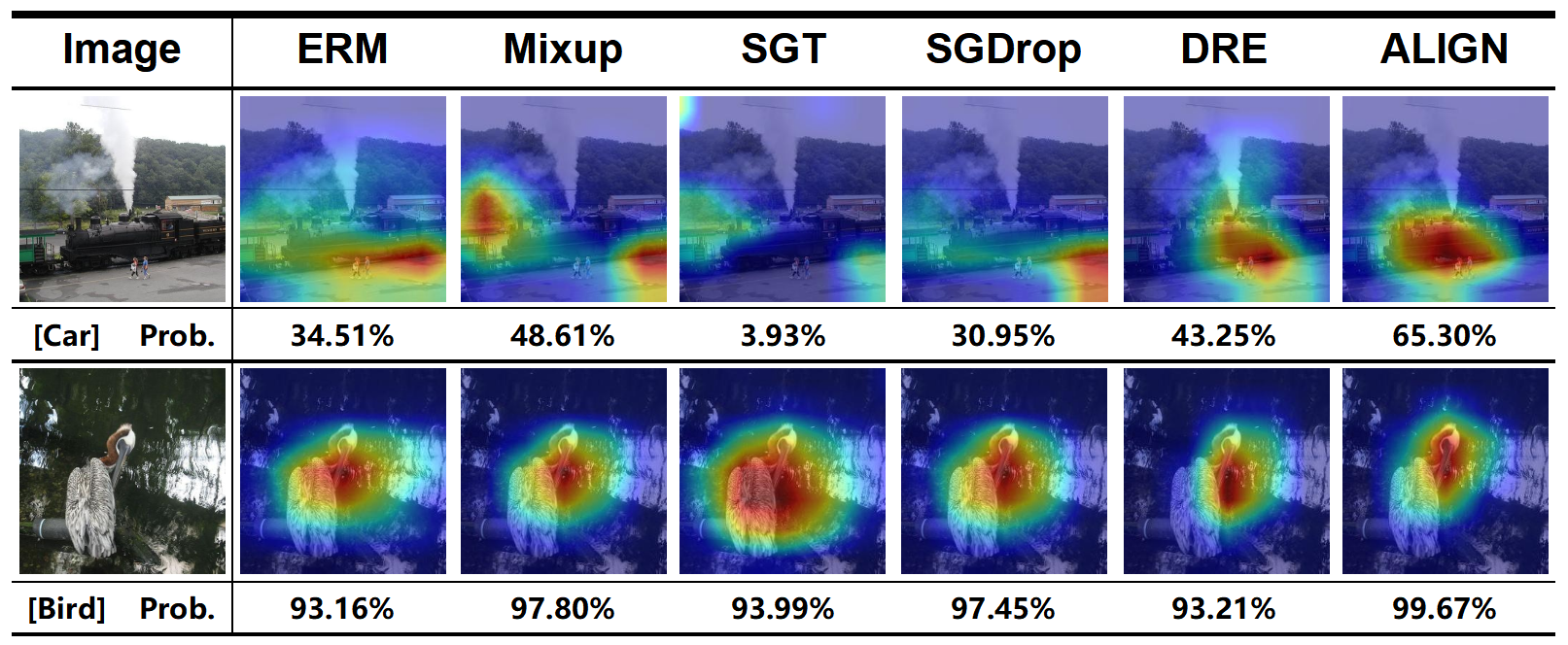}
    \caption{Two case studies from the VLCS dataset.}
    % model's prediction for that label is reported as well.}
    \label{fig:case_study}
\end{figure}

In the first example (ground-truth label: \texttt{car}), baselines exhibit scattered and misaligned attention, often focusing on irrelevant background regions (e.g., sky, smoke, or vegetation). This diluted focus results in substantially lower classification confidence. 
In contrast, ALIGN concentrates its attention around the vehicle, supporting a more accurate prediction (65.30\%, the highest among all methods).

In the second example (ground-truth label: \texttt{bird}), 
most baselines attend broadly to the surrounding area, incorporating unnecessary contextual features (e.g., background trees or shadows). 
ALIGN, however, focuses consistently on the bird’s highly discriminative regions (head and beak), leading to the highest prediction confidence (99.67\%).

% In the first example, other models exhibit scattered attention, focusing on irrelevant background regions, diluting the model’s reasoning. 
% In stark contrast, ALIGN maintains sharply concentrated attention on the primary object, effectively filtering out background noise and thereby supporting more accurate and semantically meaningful predictions.

% In the second example, ERM, Mixup, and SGT mistakenly incorporate unrelated contextual elements when recognizing the bird, resulting in less reliable identification. DRE, while better, focuses primarily on the bird’s beak, showing limited localization precision. By comparison, ALIGN sustains precise and consistent attention on both the bird’s head and beak, ensuring more robust and reliable identification.

Overall, two real examples highlight ALIGN’s ability to consistently ground its predictions in semantically meaningful regions, leading to superior performance in both accuracy and interpretability relative to baseline models.

\paragraph{Mask Evaluation.}
To assess mask quality, we conducted a qualitative study using 100 randomly selected cases from each VLCS sub-dataset, including outputs from SAM and our masker.
Four independent volunteers performed pairwise comparisons, labeling each case as Win, Tie, or Lose based on the masker’s performance against SAM. 
Additionally, representative heatmaps were visualized to compare the spatial attention and segmentation consistency of both methods.
Results are shown in Fig.~\ref{fig:case_mask}.

\begin{figure}[h]
    \centering
    \includegraphics[width=0.9\linewidth]{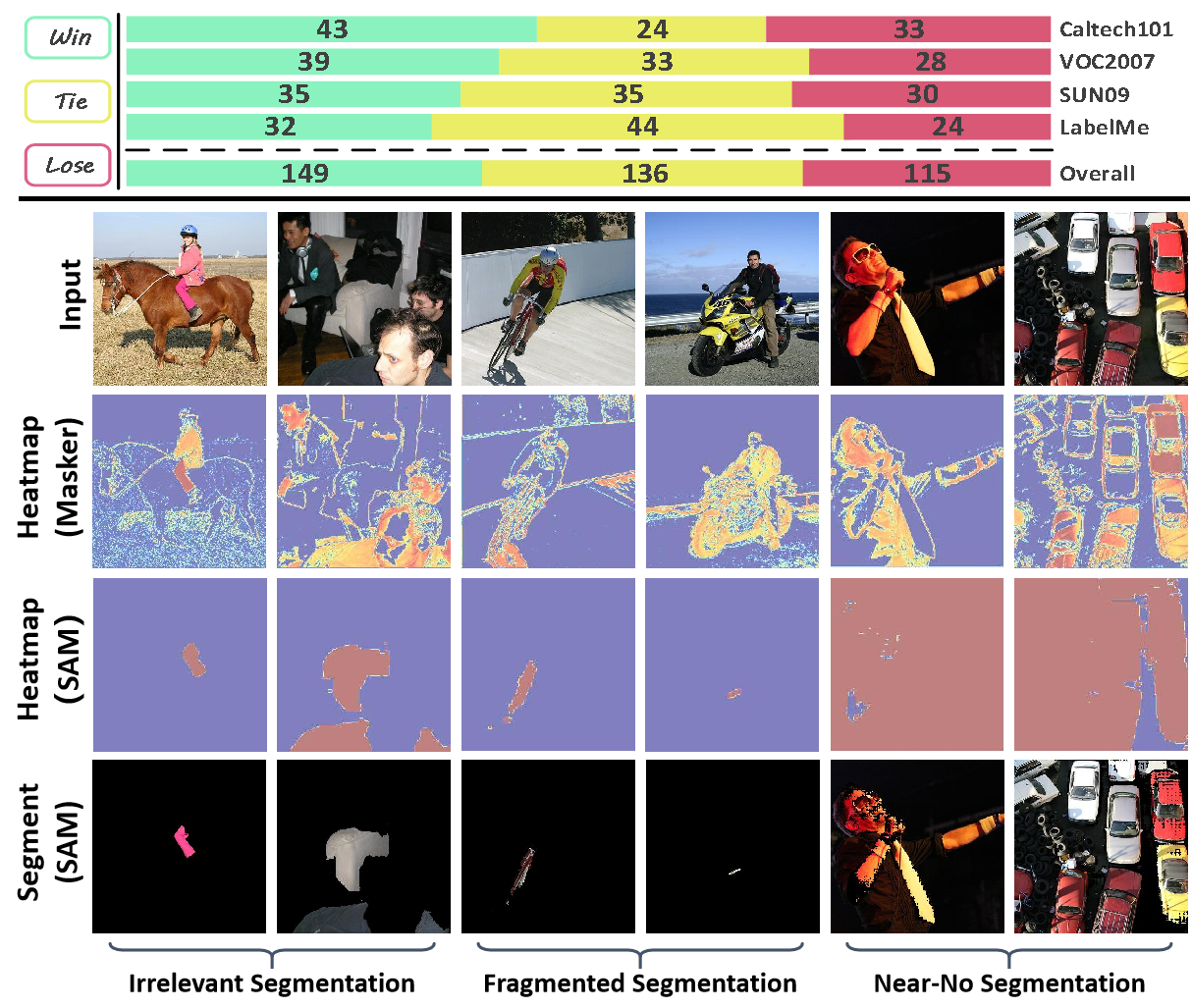}
    \caption{Mask quality evaluation results on VLCS dataset. 
    Upper: human assessment; Lower: visualization examples.}
    \label{fig:case_mask}
\end{figure}

From the human evaluation, our masker consistently outperformed SAM, especially on \textbf{Caltech101} and \textbf{VOC2007}, achieving 149 wins versus 115 losses (\textbf{Win/Lose ratio $\approx$ 1.30}), demonstrating clear superiority. 
Visualization analysis further revealed that SAM often produced (1) \textit{irrelevant}, (2) \textit{fragmented}, or (3) \textit{near-empty} segmentations, while our masker precisely highlighted task-relevant target regions. 
These results suggest that SAM, being \textbf{task-agnostic}, generates suboptimal masks when used directly as guidance, whereas our \textbf{task-specific} masker yields more reliable and semantically aligned masks for downstream applications.

\section{Conclusion}
We propose ALIGN, a novel EGL framework that iteratively trains a masker alongside the classifier, effectively constraining the model’s attention to object-relevant regions.
This end-to-end approach not only enhances performance under domain shifts, but also yields more faithful and interpretable explanations, as demonstrated by both quantitative and qualitative evaluations. 
In future work, we plan to extend ALIGN to multi-object scenarios and explore alternative explanation mechanisms to further enrich interpretability.

\section*{Acknowledgments}
This work was supported by the National Natural Science Foundation of China (No. 62476060). 
Chao Chen was also supported by the National Key Research and Development Program of China (No. 2023YFB3106504) and Pengcheng-China Mobile Jointly Funded Project (No.
2024ZY2B0050).
We appreciate all the co-workers' constructive comments, which significantly contributed to the development of this paper.

% \newpage
\bibliography{aaai2026}

@inproceedings{guesmi2024asgt,
  title={Exploring the Interplay of Interpretability and Robustness in Deep Neural Networks: A Saliency-Guided Approach},
  author={Guesmi, Amira and Aswani, Nishant Suresh and Shafique, Muhammad},
  booktitle={2024 IEEE International Conference on Image Processing Challenges and Workshops (ICIPCW)},
  pages={4066--4072},
  year={2024},
  organization={IEEE}
}

@inproceedings{li2023dre,
  title={Are Data-driven Explanations Robust against Out-of-distribution Data?},
  author={Li, Tang and Qiao, Fengchun and Ma, Mengmeng and Peng, Xi},
  booktitle={Proceedings of the IEEE/CVF Conference on Computer Vision and Pattern Recognition},
  pages={3821--3831},
  year={2023}
}

@inproceedings{rieger2020cdep,
  title={Interpretations are useful: penalizing explanations to align neural networks with prior knowledge},
  author={Rieger, Laura and Singh, Chandan and Murdoch, William and Yu, Bin},
  booktitle={International conference on machine learning},
  pages={8116--8126},
  year={2020},
  organization={PMLR}
}

@article{ross2017right,
  title={Right for the right reasons: Training differentiable models by constraining their explanations},
  author={Ross, Andrew Slavin and Hughes, Michael C and Doshi-Velez, Finale},
  journal={arXiv preprint arXiv:1703.03717},
  year={2017}
}

@article{ismail2021sgt,
  title={Improving deep learning interpretability by saliency guided training},
  author={Ismail, Aya Abdelsalam and Corrada Bravo, Hector and Feizi, Soheil},
  journal={Advances in Neural Information Processing Systems},
  volume={34},
  pages={26726--26739},
  year={2021}
}

@article{bertoin2024sgdrop,
  title={The Overfocusing Bias of Convolutional Neural Networks: A Saliency-Guided Regularization Approach},
  author={Bertoin, David and Sanchez, Eduardo Hugo and Zouitine, Mehdi and Rachelson, Emmanuel},
  journal={arXiv preprint arXiv:2409.17370},
  year={2024}
}

@article{guidedlearning,
  title={Going beyond xai: A systematic survey for explanation-guided learning},
  author={Gao, Yuyang and Gu, Siyi and Jiang, Junji and Hong, Sungsoo Ray and Yu, Dazhou and Zhao, Liang},
  journal={ACM Computing Surveys},
  year={2024},
}

@inproceedings{VLCS,
  title={Unbiased metric learning: On the utilization of multiple datasets and web images for softening bias},
  author={Fang, Chen and Xu, Ye and Rockmore, Daniel N},
  booktitle={ICCV},
  year={2013}
}

@inproceedings{he2016resnet,
  title={Deep residual learning for image recognition},
  author={He, Kaiming and Zhang, Xiangyu and Ren, Shaoqing and Sun, Jian},
  booktitle={Proceedings of the IEEE conference on computer vision and pattern recognition},
  pages={770--778},
  year={2016}
}

@article{tsipras2018robustness,
  title={Robustness may be at odds with accuracy},
  author={Tsipras, Dimitris and Santurkar, Shibani and Engstrom, Logan and Turner, Alexander and Madry, Aleksander},
  journal={arXiv preprint arXiv:1805.12152},
  year={2018}
}

@article{selvaraju2020gradcam,
  title={Grad-CAM: visual explanations from deep networks via gradient-based localization},
  author={Selvaraju, Ramprasaath R and Cogswell, Michael and Das, Abhishek and Vedantam, Ramakrishna and Parikh, Devi and Batra, Dhruv},
  journal={International journal of computer vision},
  volume={128},
  pages={336--359},
  year={2020},
  publisher={Springer}
}

@inproceedings{zhang2023magi,
  title={Magi: Multi-annotated explanation-guided learning},
  author={Zhang, Yifei and Gu, Siyi and Gao, Yuyang and Pan, Bo and Yang, Xiaofeng and Zhao, Liang},
  booktitle={Proceedings of the IEEE/CVF International Conference on Computer Vision},
  pages={1977--1987},
  year={2023}
}

@inproceedings{zhuang2019care,
  title={Care: Class attention to regions of lesion for classification on imbalanced data},
  author={Zhuang, Jiaxin and Cai, Jiabin and Wang, Ruixuan and Zhang, Jianguo and Zheng, Weishi},
  booktitle={International Conference on Medical Imaging with Deep Learning},
  pages={588--597},
  year={2019},
  organization={PMLR}
}

@inproceedings{xu2020adversarialmixup,
  title={Adversarial domain adaptation with domain mixup},
  author={Xu, Minghao and Zhang, Jian and Ni, Bingbing and Li, Teng and Wang, Chengjie and Tian, Qi and Zhang, Wenjun},
  booktitle={Proceedings of the AAAI conference on artificial intelligence},
  volume={34},
  number={04},
  pages={6502--6509},
  year={2020}
}

@article{arjovsky2019irm,
  title={Invariant risk minimization},
  author={Arjovsky, Martin and Bottou, L{\'e}on and Gulrajani, Ishaan and Lopez-Paz, David},
  journal={arXiv preprint arXiv:1907.02893},
  year={2019}
}

@article{srinivas2023models,
  title={Which models have perceptually-aligned gradients? an explanation via off-manifold robustness},
  author={Srinivas, Suraj and Bordt, Sebastian and Lakkaraju, Himabindu},
  journal={Advances in neural information processing systems},
  volume={36},
  pages={21172--21195},
  year={2023}
}

@article{gao2022gradia,
  title={Aligning eyes between humans and deep neural network through interactive attention alignment},
  author={Gao, Yuyang and Sun, Tong Steven and Zhao, Liang and Hong, Sungsoo Ray},
  journal={Proceedings of the ACM on Human-Computer Interaction},
  volume={6},
  number={CSCW2},
  pages={1--28},
  year={2022},
  publisher={ACM New York, NY, USA}
}

@inproceedings{gao2022res,
  title={Res: A robust framework for guiding visual explanation},
  author={Gao, Yuyang and Sun, Tong Steven and Bai, Guangji and Gu, Siyi and Hong, Sungsoo Ray and Liang, Zhao},
  booktitle={proceedings of the 28th ACM SIGKDD conference on knowledge discovery and data mining},
  pages={432--442},
  year={2022}
}

@inproceedings{karkehabadi2024smoot,
  title={SMOOT: Saliency guided mask optimized online training},
  author={Karkehabadi, Ali and Homayoun, Houman and Sasan, Avesta},
  booktitle={2024 IEEE 17th Dallas circuits and systems conference (DCAS)},
  pages={1--6},
  year={2024},
  organization={IEEE}
}

@article{ying2022visfis,
  title={Visfis: Visual feature importance supervision with right-for-the-right-reason objectives},
  author={Ying, Zhuofan and Hase, Peter and Bansal, Mohit},
  journal={Advances in Neural Information Processing Systems},
  volume={35},
  pages={17057--17072},
  year={2022}
}

@inproceedings{wang2024gazegnn,
  title={Gazegnn: A gaze-guided graph neural network for chest x-ray classification},
  author={Wang, Bin and Pan, Hongyi and Aboah, Armstrong and Zhang, Zheyuan and Keles, Elif and Torigian, Drew and Turkbey, Baris and Krupinski, Elizabeth and Udupa, Jayaram and Bagci, Ulas},
  booktitle={Proceedings of the IEEE/CVF Winter Conference on Applications of Computer Vision},
  pages={2194--2203},
  year={2024}
}

@article{etemadyrad2024ggnes,
  title={Global explanation supervision for Graph Neural Networks},
  author={Etemadyrad, Negar and Gao, Yuyang and Manoj Pudukotai Dinakarrao, Sai and Zhao, Liang},
  journal={Frontiers in big Data},
  volume={7},
  pages={1410424},
  year={2024},
  publisher={Frontiers Media SA}
}

@inproceedings{gao2021gnes,
  title={Gnes: Learning to explain graph neural networks},
  author={Gao, Yuyang and Sun, Tong and Bhatt, Rishab and Yu, Dazhou and Hong, Sungsoo and Zhao, Liang},
  booktitle={2021 IEEE international conference on data mining (ICDM)},
  pages={131--140},
  year={2021},
  organization={IEEE}
}

@article{li2023symbolic,
  title={Symbolic chain-of-thought distillation: Small models can also" think" step-by-step},
  author={Li, Liunian Harold and Hessel, Jack and Yu, Youngjae and Ren, Xiang and Chang, Kai-Wei and Choi, Yejin},
  journal={arXiv preprint arXiv:2306.14050},
  year={2023}
}

@article{li2022explanations,
  title={Explanations from large language models make small reasoners better},
  author={Li, Shiyang and Chen, Jianshu and Shen, Yelong and Chen, Zhiyu and Zhang, Xinlu and Li, Zekun and Wang, Hong and Qian, Jing and Peng, Baolin and Mao, Yi and others},
  journal={arXiv preprint arXiv:2210.06726},
  year={2022}
}

@article{kirillov2023segany,
  title={Segment Anything},
  author={Kirillov, Alexander and Mintun, Eric and Ravi, Nikhila and Mao, Hanzi and Rolland, Chloe and Gustafson, Laura and Xiao, Tete and Whitehead, Spencer and Berg, Alexander C. and Lo, Wan-Yen and Doll{\'a}r, Piotr and Girshick, Ross},
  journal={arXiv:2304.02643},
  year={2023}
}

@book{cover1999elements,
  title={Elements of information theory},
  author={Cover, Thomas M},
  year={1999},
  publisher={John Wiley \& Sons}
}

@article{ruby2020binary,
  title={Binary cross entropy with deep learning technique for image classification},
  author={Ruby, Usha and Yendapalli, Vamsidhar and others},
  journal={Int. J. Adv. Trends Comput. Sci. Eng},
  volume={9},
  number={10},
  year={2020}
}

\newpage

\newpage
\null
\newpage

\renewcommand{\thesubsection}{\thesection.\arabic{subsection}}

\appendix

\section{Technical Appendix}

\subsection{Proofs}
\label{appendix:proofs}

\begin{proof}[Proof of Lemma \ref{lemma:mse_discrepancy}]
We first expand the inner part and compute the difference between source and target domains:
\begin{equation}
\begin{aligned}
\Delta_{MSE} &= \left| \mathbb{E}_{\mathcal{D}_T}[(f(x) - y)^2] - \mathbb{E}_{\mathcal{D}_S}[(f(x) - y)^2] \right|\\
&= \Bigl| \left( \mathbb{E}_{\mathcal{D}_T}[f(x)^2] - \mathbb{E}_{\mathcal{D}_S}[f(x)^2] \right) \\
& \quad
- 2 \left( \mathbb{E}_{\mathcal{D}_T}[y f(x)] - \mathbb{E}_{\mathcal{D}_S}[y f(x)] \right) \\
& \quad 
+ \left( \mathbb{E}_{\mathcal{D}_T}[y^2] - \mathbb{E}_{\mathcal{D}_S}[y^2] \right) \Bigr| \\
& \leq \left| \mathbb{E}_{\mathcal{D}_T}[f(x)^2] - \mathbb{E}_{\mathcal{D}_S}[f(x)^2] \right| \\
& \quad
+ 2 \left| \mathbb{E}_{\mathcal{D}_T}[y f(x)] - \mathbb{E}_{\mathcal{D}_S}[y f(x)] \right| \\
& \quad
+ \left| \mathbb{E}_{\mathcal{D}_T}[y^2] - \mathbb{E}_{\mathcal{D}_S}[y^2] \right| 
\end{aligned}
\end{equation}

Since $|y|, |f(x)| \le 1$, then $|y f(x)| \le 1$, and based on Cauchy-Schwarz inequality:
\begin{equation}
% \begin{aligned}
    \left| \mathbb{E}_{\mathcal{D}_T}[f(x)^2] - \mathbb{E}_{\mathcal{D}_S}[f(x)^2] \right|
    % &
    \le 2 \left| \mathbb{E}_{\mathcal{D}_T}[f(x)] - \mathbb{E}_{\mathcal{D}_S}[f(x)] \right|, 
    % \\
    % \left| \mathbb{E}_{\mathcal{D}_T}[y f(x)] - \mathbb{E}_{\mathcal{D}_S}[y f(x)] \right| &\le \left| \mathbb{E}_{\mathcal{D}_T}[f(x)] - \mathbb{E}_{\mathcal{D}_S}[f(x)] \right|.
% \end{aligned}
\end{equation}

Now bound the term involving $y f(x)$ using Cauchy-Schwarz and the assumption $|y| \le 1$:
\begin{equation}
\left| \mathbb{E}_{\mathcal{D}_T}[y f(x)] - \mathbb{E}_{\mathcal{D}_S}[y f(x)] \right|
\le \left| \mathbb{E}_{\mathcal{D}_T}[f(x)] - \mathbb{E}_{\mathcal{D}_S}[f(x)] \right|.
\end{equation}

% Also, by assuming bounded outputs $|f(x)| \leq 1$, we have:
% % $\left| \mathbb{E}_{\mathcal{D}_T}[f(x)] + \mathbb{E}_{\mathcal{D}_S}[f(x)] \right| \leq 2$:
% \begin{equation}
% \left| \mathbb{E}_{\mathcal{D}_T}[f(x)^2] - \mathbb{E}_{\mathcal{D}_S}[f(x)^2] \right|
% \le 2 \left| \mathbb{E}_{\mathcal{D}_T}[f(x)] - \mathbb{E}_{\mathcal{D}_S}[f(x)] \right|.
% \end{equation}

Summing terms yields the following bound:
\begin{equation}
\begin{aligned}
\Delta_{MSE} & \leq \left| \mathbb{E}_{\mathcal{D}_T}[f(x)^2] - \mathbb{E}_{\mathcal{D}_S}[f(x)^2] \right| \\
& \quad
+ 2 \left| \mathbb{E}_{\mathcal{D}_T}[y f(x)] - \mathbb{E}_{\mathcal{D}_S}[y f(x)] \right| 
+ \left| \mathbb{E}_{\mathcal{D}_T}[y^2] - \mathbb{E}_{\mathcal{D}_S}[y^2] \right| \\
% &\le 2 \left| \mathbb{E}_{\mathcal{D}_T}[f(x)] - \mathbb{E}_{\mathcal{D}_S}[f(x)] \right|
% + 2 \left| \mathbb{E}_{\mathcal{D}_T}[f(x)] - \mathbb{E}_{\mathcal{D}_S}[f(x)] \right|
% + \left| \mathbb{E}_{\mathcal{D}_T}[y^2] - \mathbb{E}_{\mathcal{D}_S}[y^2] \right| \\
% &= 
& \leq
4 \left| \mathbb{E}_{\mathcal{D}_T}[f(x)] - \mathbb{E}_{\mathcal{D}_S}[f(x)] \right|
+ \left| \mathbb{E}_{\mathcal{D}_T}[y^2] - \mathbb{E}_{\mathcal{D}_S}[y^2] \right|.
\end{aligned}
\end{equation}

If the label variance remains unchanged between domains, i.e., $\mathbb{E}_{\mathcal{D}_T}[y^2] = \mathbb{E}_{\mathcal{D}_S}[y^2]$, the bound simplifies to the final claim.
\end{proof}

\begin{proof}[Proof Sketch of Lemma \ref{lemma:ce_discrepancy}]

The cross-entropy loss for a single example can be written as $CE(f(x),y) = -\log f(x)_y$. 
The function $q \mapsto -\log q$ is continuously differentiable and its derivative is bounded on $(0,M)$ if $q$ is bounded away from 0 by some $1/M$ ($M\ll 1)$. 
In other words, $-\log f(x)_y$ is Lipschitz continuous on the range of $f(x)_y$, which implies that a change of $\epsilon$ in $f(x)_y$ causes at most $M\epsilon$ change in $-\log f(x)_y$. 
% In our case, $q = f(x)_y$ changes by at most $\epsilon$ between domains for each $(x,y)$ by assumption. Therefore $|\ell_S(f(x),y)-\ell_T(f(x),y)| \leq M\epsilon$ pointwise. 
Taking expectation over $(x,y)$ yields the bound. 
% For reasonably well-calibrated classifiers $M$ is not large; if not, one can restrict to a probability range $[\delta,1-\delta]$ to avoid the singularity at 0.
\end{proof}

\begin{proof}[Proof of Lemma \ref{lemma:in_domain_error_feature_inclusion}]
We first give proofs concerning the MSE cases.
By the law of total variance:
\begin{align*}
\mathrm{Var}(y \mid x^{(sub)}) &= \mathbb{E}\left[ \mathrm{Var}(y \mid x^{(obj)}) \mid x^{(sub)} \right] \\
&+ \mathrm{Var}\left( \mathbb{E}[y \mid x^{(obj)}] \mid x^{(sub)} \right).
\end{align*}

Taking expectations with respect to $x^{(sub)}$:
\begin{align*}
\mathbb{E}_{x^{(sub)}}[\mathrm{Var}(y \mid x^{(sub)})] 
&= \mathbb{E}_{x^{(sub)}}\left[ \mathbb{E}\left[ \mathrm{Var}(y \mid x^{(obj)}) \mid x^{(sub)} \right] \right] \\
& + \mathbb{E}_{x^{(sub)}}\left[\mathrm{Var}\left( \mathbb{E}[y \mid x^{(obj)}] \mid x^{(sub)} \right) \right].
\end{align*}

By the tower property of expectation:
\[
\mathbb{E}_{x^{(sub)}}\left[ \mathbb{E}\left[ \mathrm{Var}(y \mid x^{(obj)}) \mid x^{(sub)} \right] \right] = \mathbb{E}_{x^{(obj)}}[\mathrm{Var}(y \mid x^{(obj)})].
\]

Thus:
\[
\mathbb{E}_{x^{(sub)}}[\mathrm{Var}(y \mid x^{(sub)})] \geq \mathbb{E}_{x^{(obj)}}[\mathrm{Var}(y \mid x^{(obj)})].
\]

Since for squared loss, the Bayes risk equals to the expected conditional variance:
\begin{align*}
    \mathbb{E}\left[(y - f_2^*(x^{(obj)}))^2\right] &= \mathbb{E}[\mathrm{Var}(y \mid x^{(obj)})], \\
    \mathbb{E}\left[(y - f_3^*(x^{(sub)}))^2\right] &= \mathbb{E}[\mathrm{Var}(y \mid x^{(sub)})],
\end{align*}
we conclude that:
\[
\mathbb{E}\left[(y - f_2^*(x^{(obj)}))^2\right]
\leq 
\mathbb{E}\left[(y - f_3^*(x^{(sub)}))^2\right].
\]

Equality holds if and only if:
\[
\mathrm{Var}\left( \mathbb{E}[y \mid x^{(obj)}] \mid x^{(sub)} \right) = 0,
\]
which implies that \( \mathbb{E}[y \mid x^{(obj)}] = \mathbb{E}[y \mid x^{(sub)}] \), i.e., the additional features in $x^{(obj)} \setminus x^{(sub)}$ do not provide extra information about $y$ given $x^{(sub)}$. Therefore, strict inequality holds if there exists $x_j \in x^{(obj)} \setminus x^{(sub)}$ such that $x_j \not\!\perp\!\!\!\perp y \mid x^{(sub)}$.

Now consider the cross entropy case.
The expected cross-entropy loss of a Bayes-optimal classifier equals to the conditional entropy:
\begin{align*}
\mathbb{E}_{\mathcal{D}_S}[-\log f_2(x)_y] 
& = \mathbb{E}_{x \sim \mathcal{D}_S} \left[ \mathbb{E}_{y \sim p(y \mid x^{obj})} [-\log p(y \mid x^{(obj)})] \right] 
\\
&= 
H(y \mid x^{(obj)}),
\end{align*}
and similar for $f_3$ relying on $x^{(sub)}$.

Since $x^{(sub)} \subsetneq x^{(obj)}$, 
the data-processing inequality for entropy \cite{cover1999elements} gives: 
\begin{equation*}
H(y \mid x^{(obj)}) \leq H(y \mid x^{(sub)}),
\end{equation*}
with equality if and only if:
\[
y \perp \{x^{(obj)} \setminus x^{(sub)}\} \mid x^{(sub)}.
\]

Therefore, the expected cross-entropy loss satisfies:
\[
\mathbb{E}_{\mathcal{D}_S}[-\log f_2(x)_y] \leq \mathbb{E}_{\mathcal{D}_S}[-\log f_3(x)_y],
\]
with equality only under the stated conditional independence.

\end{proof}

\begin{table*}[h]
    \footnotesize
    \setlength{\tabcolsep}{1mm}
    \centering
    \begin{tabular}{c|cc|cc|cc|cc||cc|cc|cc|cc}
    \toprule
        Metric & \multicolumn{8}{c||}{Acc} & \multicolumn{8}{c}{AUC} \\ \midrule
        \multirow{2}{*}{Method}
        & \multicolumn{2}{c|}{VOC2007} & \multicolumn{2}{c|}{LabelMe} & \multicolumn{2}{c|}{Caltech101} & \multicolumn{2}{c||}{SUN09} 
        & \multicolumn{2}{c|}{VOC2007} & \multicolumn{2}{c|}{LabelMe} & \multicolumn{2}{c|}{Caltech101} & \multicolumn{2}{c}{SUN09}
        
        \\ \cmidrule{2-17}
        & SAM & masker & SAM & masker & SAM & masker & SAM & masker & SAM & masker & SAM & masker & SAM & masker & SAM & masker \\ \midrule
        ERM & 80.00 & 85.92 & 75.09 & 81.13 & 99.64 & 99.65 & 81.40 & 81.86 & 74.19 & 75.54 & 70.40 & 71.31 & 97.19 & 97.27 & 75.11 & 76.37 \\ 
        SGT & 83.85 & 86.37 & 75.84 &  79.24 & 99.29 & 99.65 & 77.28 & 80.03 & 70.56 & 70.76 & 55.28 & 59.98 & 95.18 & 95.66 & 61.93 & 62.44 \\ 
        SGDrop & 78.66 & 84.44 & 71.88 &  79.05 & 100 &100 & 78.20 &  79.57 & 75.00 & 77.52 & 65.99 & 68.00 & 97.93 & 98.16 & 74.35 & 75.06 \\ 
        DRE & 84.00 & 86.81 & 73.58 & 80.00 & 100 & 100 & 80.94 & 83.53 & 75.28 & 77.21 & 71.87 & 73.51 & 98.87 & 99.38 & 69.19 & 71.28 \\ 
        ALIGN & 83.40 & 86.81 & 77.35 & 82.26 & 100 &  100 & 80.94 & 86.81 & 81.11 & 81.37 & 72.65 & 73.85 & 99.21 & 99.32 & 68.15 & 81.37 \\ \bottomrule
    \end{tabular}
    \caption{Numerical results on impct of mask quality during inference on VLCS dataset.}
    \label{tab:bgrm}
\end{table*}

\subsection{Related work}
\label{appendix:related_works}
\textbf{Human-Annotated Saliency Supervision.}
These methods use human annotations to guide models in focusing on semantically relevant regions.
Specifically, RES \cite{gao2022res} introduces a saliency-guided loss for explanation invariance. 
MAGI \cite{zhang2023magi} uses multi-annotator masks to align model saliency with diverse perspectives. 
CARE \cite{zhuang2019care} integrates attention masks for lesion classification in imbalanced medical imaging tasks.
GRADIA \cite{gao2022gradia} supports interactive human–model alignment for attention-based explanations.
VISFIS \cite{ying2022visfis} shows that incorporating visual feature-importance supervision and Right-for-the-Right-Reason objectives benefits in VQA tasks.
Although effective, these methods depend heavily on \textbf{manually annotated} masks or insertion/deletion metrics, which typically requires high cost and thus limits scalability.

\textbf{Human-Annotated-Free Approaches.}
Without relying on human annotations, these methods use saliency cues to guide model learning.
SGT \cite{ismail2021sgt} uses gradient-based input masking to enforce consistency between the masked and original predictions.
SMOOT \cite{karkehabadi2024smoot} further presents it as part of the online training, dynamically determining which pixels to occlude to promote more meaningful saliency learning. 
DRE \cite{li2023dre} stabilizes the explanations under distribution shifts. 
These approaches eliminate supervised masks, but often introduce additional hyper-parameters and \textit{lack theoretical analysis}.

\textbf{Beyond Image tasks.}
There are some works to explore EGL within \textit{graph neural networks}. GazeGNN \cite{wang2024gazegnn} improves interpretability and robustness in GNNs by aligning learned attention with human-understandable node importance. 
GNES \cite{gao2021gnes} and GG-NES \cite{etemadyrad2024ggnes} leverage explanation signals to guide GNN training, demonstrating explanation consistency improves performance under distribution shifts in graph-related tasks. 
EGL has also been applied to \textit{large language models} to improve reasoning capabilities while enhancing transparency. 
Works like \cite{li2023symbolic,li2022explanations} explore using explanations from LLMs as supervision signals to train smaller-scale models. 
These approaches leverage chain-of-thought explanations or attention-based rationales produced by LLMs to guide the training of compact models. 
Together, these studies suggest that EGL is emerging as a general paradigm across modalities, enabling the models to learn “for the right reasons” in a scalable and effective manner.

\begin{table*}[h]
    \small
    \centering
    \begin{tabular}{c|c|cc|cc|cc|cc|cc}
    \toprule
        \multirow{2}{*}{\textbf{Train on}} & \multirow{2}{*}{\textbf{Test on}} & \multicolumn{2}{c|}{\textbf{ERM}} & \multicolumn{2}{c|}{\textbf{SGT}} & \multicolumn{2}{c|}{\textbf{SGDrop}} & \multicolumn{2}{c|}{\textbf{DRE}} & \multicolumn{2}{c}{\textbf{ALIGN}} \\ 
        % \cline{3-12}
        \cmidrule{3-12}
        & & Acc & AUC & Acc & AUC & Acc & AUC & Acc & AUC & Acc & AUC \\ \midrule
        \multirow{3}{*}{VOC2007} 
        & LabelMe & 59.43 & 80.55 & \textbf{62.12} & 77.96 & 59.81 & 77.72 & 51.71 & \textbf{83.68} & 56.89 & 82.27 \\ 
        & Caltech101 & 98.31 & 90.03 & 99.48 & 89.72 & 98.73 & 93.88 & 99.52 & 94.90 & \textbf{99.53} & \textbf{96.34} \\ 
        & SUN09 & 73.02 & 65.12 & 69.23 & 62.84 & 71.91 & 65.47 & 73.74 & 66.61 & \textbf{74.03} & \textbf{67.26} \\ \midrule
        \multirow{3}{*}{LabelMe} 
        & VOC2007 & 67.27 & \textbf{71.65} & 70.12 & 60.36 & 61.39 & 58.23 & 58.53 & 67.80 & \textbf{73.63} & 69.09 \\ 
        & Caltech101 & 90.04 & 70.60 & 91.92 & 63.69 & 89.92 & 57.66 & 89.59 & 71.79 & \textbf{96.63} & \textbf{76.11} \\ 
        & SUN09 & 56.36 & \textbf{62.63} & 59.58 & 54.58 & 50.03 & 52.88 & 51.35 & 60.37 & \textbf{61.36} & 57.85 \\ \midrule
        \multirow{3}{*}{Caltech101} 
        & VOC2007 & 51.68 & 72.99 & 41.53 & 72.73 & 42.10 & 72.15 & \textbf{53.11} & 73.81 & 49.60 & \textbf{74.05} \\ 
        & LabelMe & 38.89 & 61.93 & 33.68 & 66.11 & 33.32 & 56.99 & 41.64 & 64.61 & \textbf{42.24} & \textbf{66.41} \\ 
        & SUN09 & 42.03 & 70.43 & 34.21 & 69.21 & 37.61 & \textbf{75.76} & 43.87 & 72.43 & \textbf{48.04} & 69.85 \\ \midrule
        \multirow{3}{*}{SUN09} 
        & VOC2007 & 64.65 & \textbf{75.54} & 66.56 & 69.11 & 63.51 & 71.05 & \textbf{66.76} & 67.29 & 61.01 & 64.97 \\ 
        & LabelMe & 59.66 & 79.89 & 57.86 & 67.08 & 52.82 & 78.76 & 57.78 & 77.28 & \textbf{62.07} & \textbf{82.04} \\ 
        & Caltech101 & 73.37 & \textbf{79.89} & 73.97 & 72.14 & 65.08 & 64.37 & \textbf{83.24} & 68.74 & 65.77 & 57.43 \\ \bottomrule
    \end{tabular}
    \caption{Experimental results on VLCS with the OOD setting.}
    \label{tab:ood-vlcs}
\end{table*}

\begin{table*}[h]
    \small
    \centering
    \begin{tabular}{c|c|cc|cc|cc|cc|cc}
    \toprule
        \multirow{2}{*}{\textbf{Train on}} & \multirow{2}{*}{\textbf{Test on}} & \multicolumn{2}{c|}{\textbf{ERM}} & \multicolumn{2}{c|}{\textbf{SGT}} & \multicolumn{2}{c|}{\textbf{SGDrop}} & \multicolumn{2}{c|}{\textbf{DRE}} & \multicolumn{2}{c}{\textbf{ALIGN}} \\ 
        % \cline{3-12}
        \cmidrule{3-12}
        & & Acc & AUC & Acc & AUC & Acc & AUC & Acc & AUC & Acc & AUC \\ \midrule
        \multirow{3}{*}{Loc\_38} 
        & Loc\_43 & 9.46 & 39.58 & \textbf{18.35} & 37.61 & 12.85 & 38.06 & 17.76 & \textbf{42.75} & 13.21 & 41.13 \\ 
        & Loc\_46 & 11.62 & 32.99 & 18.68 & 34.61 & \textbf{19.17} & \textbf{42.80} & 16.75 & 35.69 & 14.57 & 40.86 \\ 
        & Loc\_100 & 18.75 & 38.18 & 25.02 & 29.74 & 39.34 & 42.48 & 29.10 & 41.88 & \textbf{43.01} & \textbf{49.04} \\ \midrule
        \multirow{3}{*}{Loc\_43} 
        & Loc\_38 & 40.75 & 54.55 & \textbf{42.29} & \textbf{60.42} & 40.62 & 52.87 & 35.08 & 49.84 & 39.24 & 48.82 \\ 
        & Loc\_46 & 40.70 & 46.82 & 40.68 & 56.03 & 39.00 & 47.03 & \textbf{45.11} & 60.04 & 37.32 & \textbf{60.07} \\ 
        & Loc\_100 & 38.51 & 46.29 & \textbf{51.33} & 62.35 & 46.19 & 47.23 & 50.91 & 53.50 & 29.46 & \textbf{62.96} \\ \midrule
        \multirow{3}{*}{Loc\_46} 
        & Loc\_38 & 22.00 & \textbf{49.30} & 28.45 & 47.46 & \textbf{32.15} & 40.15 & 23.34 & 38.64 & 27.28 & 46.79 \\ 
        & Loc\_43 & 46.53 & 45.85 & 45.85 & 43.92 & 46.42 & 52.37 & \textbf{49.05} & 46.53 & 48.83 & \textbf{57.40} \\ 
        & Loc\_100 & \textbf{54.70} & \textbf{54.88} & 54.21 & 51.83 & 31.53 & 41.46 & 31.48 & 47.62 & 42.48 & 53.98 \\ \midrule
        \multirow{3}{*}{Loc\_100} 
        & Loc\_38 & 28.03 & 57.66 & 16.06 & 47.11 & \textbf{44.23} & 62.14 & 32.83 & 59.96 & 25.59 & \textbf{66.66} \\ 
        & Loc\_43 & 22.88 & \textbf{51.20} & 25.01 & 41.91 & 25.39 & 42.76 & \textbf{26.46} & 45.91 & 22.95 & 50.80 \\ 
        & Loc\_46 & \textbf{39.82} & 54.36 & 38.72 & 49.40 & 36.56 & 48.85 & 34.43 & 53.21 & 37.80 & \textbf{60.97} \\ \bottomrule
    \end{tabular}
    \caption{Experimental results on Terra Incognita with the OOD setting.}
    \label{tab:ood-loc}
\end{table*}

\subsection{Implementation detail}
\label{appendix:config}
We conduct all the experiments using PyTorch on a single NVIDIA's T4 GPU.
The classification models are built upon ResNet-18 backbones. 
The masker is a CNN-based model that consists of three convolutional blocks: the first two use 3x3 kernels, with ReLU and Batch Normalization applied after each convolution, 
% to introduce non-linearity and stabilize training, 
while the final block applies a 1x1 convolution followed by a Sigmoid activation to produce a mask. 
All images are resized to $224\times224$.

Both the classifier and the masker are optimized using the Adam optimizer.
The hyperparameters $\lambda_1$, $\lambda_2$, $\lambda_3$, and $\lambda_4$ are tuned via cross-validation.
Finally, ALIGN adopts $\lambda_1 = 10$ and $\lambda_2 = 1$ for \textit{masker} across all the experiments.
As for \textit{classifier}, $\lambda_3 = \lambda_4 = 0.1$ for VLCS dataset, and $\lambda_3 = \lambda_4 = 0.2$ for Terra Incognita dataset.
% The final hyperparameter settings for \textbf{ALIGN} are as follows:  
% \begin{itemize}
%     \item For \textit{masker}, we fix $\lambda_1 = 10$ and $\lambda_2 = 1$ across all the experiments.
%     \item For \textit{classifier}, we set $\lambda_3 = \lambda_4 = 0.1$ for VLCS dataset, and $\lambda_3 = \lambda_4 = 0.2$ for Terra Incognita dataset.
% \end{itemize}

The training of ALIGN is performed in two stages: we initially train the classifier for 200 iterations, followed by iterative training of both the classifier and the masker for an additional 100 iterations. 
Early stopping is also employed to prevent overfitting.

% The primary implementation is publicly accessible via an anonymous GitHub repository at: \url{https://anonymous.4open.science/r/align-2E3B/}.

The explanations are evaluated by the following two metrics \cite{guidedlearning}: 

\begin{itemize}

    \item \textbf{Sufficiency (Suff)} quantifies the degradation in model performance when only the salient regions of the input are retained. It is formally defined as:
    \begin{equation}
        Suff = f_y(x) - f_y(g(x, \Phi_y(x))),
    \end{equation}
    where $g(x, \Phi_y(x))$ represents the subset of input pixels deemed highly important by the explanation. 
    A \textit{lower} Suff value indicates that the explanation captures more essential information needed for accurate prediction.

    \item \textbf{Comprehensiveness (Comp)} 
    % assesses how critical the identified salient regions are to model’s decision. 
    % It 
    measures the drop in prediction confidence when the highlighted regions are removed. Formally, 
    \begin{equation}
        Comp = f_y(x) - f_y(x \setminus g(x, \Phi_y(x))),
    \end{equation}
    where a \textit{high} value indicates that the model strongly relies on the identified regions to make its prediction.
\end{itemize}

\subsection{Impact of mask quality during Inference}
\label{appendix:mask}
To the complement perliminary experiment in Sec.~\ref{sec:preliminary},
Table~\ref{tab:bgrm} presents the complete numerical results to evaluate impact of mask quality during inference on the VLCS dataset.

\subsection{Complete Experimental Results under OOD Setting}
\label{appendix:ood}
To evaluate model predicvtive performance (evaluated by accuracy and AUC) under OOD settings, 
Tables~\ref{tab:ood-vlcs} and~\ref{tab:ood-loc} report the results on the VLCS and Terra Incognita datasets, respectively.

\end{document}